%% file: currentDraft.tex
\documentclass[conference]{IEEEtran}
\IEEEoverridecommandlockouts
\usepackage{amsthm}

\newtheorem{theorem}{Theorem}
\newtheorem{lemma}{Lemma}

\newtheorem{claim}{Claim}
\newtheorem{definition}{Definition}
\newtheorem{remark}{Remark}
\usepackage{algorithm}
\usepackage{makecell}
\usepackage{cite}
\usepackage{amsmath,amssymb,amsfonts}
\usepackage{algorithmic}
\usepackage{graphicx}
\usepackage{caption}
\usepackage{subcaption}
\usepackage{textcomp}
\usepackage{xcolor}
\usepackage{mathtools}
\mathtoolsset{showonlyrefs}
\usepackage{dsfont}

\usepackage{pgfplots}
\pgfplotsset{compat=1.18}
\usepackage{tikz}

\allowdisplaybreaks

\newcommand{\btheta}{\boldsymbol{\theta}}
\newcommand{\brho}{\boldsymbol{\rho}}
\newcommand{\blambda}{\boldsymbol{\lambda}}
\newcommand{\by}{\boldsymbol{y}}
\newcommand{\be}{\boldsymbol{e}}
\newcommand{\bg}{\boldsymbol{g}}
\newcommand{\bc}{\boldsymbol{c}}
\newcommand{\bX}{\boldsymbol{X}}
\newcommand{\bY}{\boldsymbol{Y}}
\newcommand{\bZ}{\boldsymbol{Z}}
\newcommand{\bN}{\boldsymbol{N}}
\newcommand{\bw}{\boldsymbol{w}}
\DeclareMathOperator*{\argmax}{arg\,max}
\DeclareMathOperator*{\argmin}{arg\,min}

\DeclareMathOperator*{\arginf}{arginf}

\def\BibTeX{{\rm B\kern-.05em{\sc i\kern-.025em b}\kern-.08em
    T\kern-.1667em\lower.7ex\hbox{E}\kern-.125emX}}
\begin{document}

\title{Efficient Clustering in Stochastic Bandits\\
\thanks{This work is supported in part by Qualcomm University Research Grant.}
}

\author{\IEEEauthorblockN{G Dhinesh Chandran}
\IEEEauthorblockA{\textit{Department of Electrical Engineering} \\
\textit{Indian Institute of Technology Madras}\\
ee22d200@smail.iitm.ac.in}
\and
\IEEEauthorblockN{Kota Srinivas Reddy}
\IEEEauthorblockA{\textit{Department of Artificial Intelligence} \\
\textit{Indian Institute of Technology Kharagpur}\\
ksreddy@ai.iitkgp.ac.in}
\and
\IEEEauthorblockN{Srikrishna Bhashyam}
\IEEEauthorblockA{\textit{Department of Electrical Engineering} \\
\textit{Indian Institute of Technology Madras}\\
skrishna@ee.iitm.ac.in}
}

\maketitle

\begin{abstract}
We study the Bandit Clustering (BC) problem under the fixed confidence setting, where the objective is to group a collection of data sequences (arms) into clusters through sequential sampling from adaptively selected arms at each time step while ensuring a fixed error probability at the stopping time.
We consider a setting where arms in a cluster may have different distributions. Unlike existing results in this setting, which assume Gaussian-distributed arms, we study a broader class of vector-parametric distributions that satisfy mild regularity conditions. 
Existing asymptotically optimal BC algorithms require solving an optimization problem as part of their sampling rule at each step, which is computationally costly. We propose an Efficient Bandit Clustering algorithm (EBC), which, instead of solving the full optimization problem, takes a single step toward the optimal value at each time step, making it computationally efficient while remaining asymptotically optimal.
We also propose a heuristic variant of EBC, called EBC-H, which further simplifies the sampling rule, with arm selection based on quantities computed as part of the stopping rule.
We highlight the computational efficiency of EBC and EBC-H by comparing their per-sample run time with that of existing algorithms.
The asymptotic optimality of EBC is supported through simulations on the synthetic datasets. Through simulations on both synthetic and real-world datasets, we show the performance gain of EBC and EBC-H over existing approaches.

 

\end{abstract}

\vspace{-1mm}
\section{Introduction} \label{sec: intro}
Clustering a finite collection of data sequences (arms) has numerous applications \cite{yang2024optimal, jain2010data, maccuish2010clustering}. Data sequence clustering can be studied either in a fixed-sample-size setting (FSS) or a sequential setting (SEQ). In FSS, a finite sequence of data points (samples) from each arm is available a priori, and the cluster estimate is based on the available samples \cite{wang2019k, wang2020exponentially}.  In SEQ, samples from arms are available sequentially, and the algorithm is equipped with a stopping rule to decide when to stop sampling and declare the cluster estimate. Sequential (SEQ) clustering can be studied under either a full-information setting, where a sample is observed from each of the arms \cite{sreenivasan2023nonparametric, singh2025exponentially}, or a bandit setting, where a sample can be observed only from any one of the arms \cite{yang2024optimal}. We refer to data sequence clustering in the bandit setting as Bandit Clustering (BC), and an algorithm designed to solve the BC problem  as a BC algorithm. The BC algorithms, in addition to a stopping rule, are equipped with a sampling rule, which adaptively selects an arm to observe a sample. The BC problem can be studied under either the fixed-budget setting, where the number of samples is limited, or the fixed-confidence setting, where the error probability is fixed \cite{yang2024optimal}. This work focuses on the BC problem under a fixed-confidence setting.
Another class of clustering problems involves observing samples from a single data stream, where the objective is to cluster the observed data points. This problem, referred to as Online Clustering \cite{liberty2016algorithm}, differs from our setting.

The BC problem can be viewed as a special case of sequential multi-hypothesis testing \cite{deshmukh2021sequential, prabhu2022sequential}, where each possible partition of the arms corresponds to a distinct hypothesis. However, the number of hypotheses can grow exponentially, leading to high computational complexity. Notable works on the BC problem in the fixed-confidence setting include \cite{yang2024optimal, thuot2024active, yavas2025general}, which assume identical distributions for arms within each cluster. However, in practice, arms within the same cluster may follow different distributions.
The Max-Gap algorithms proposed in \cite{katariya2019maxgap} address this more general scenario but are limited to two clusters. The work in \cite{chandran2025online} considers Multivariate Gaussian-distributed arms and extends the setting of \cite{katariya2019maxgap} to more than two clusters by proposing two algorithms: Average Tracking Bandit Online Clustering (ATBOC), and Lower and Upper Confidence Bound Bandit Online Clustering (LUCBBOC).
ATBOC, which is order-wise asymptotically optimal, involves solving an optimization problem at each round, which is computationally costly. LUCBBOC is computationally efficient, but sub-optimal. 
We propose Efficient Bandit Clustering (EBC), which is both computationally efficient and asymptotically optimal. The key idea behind achieving both is that, instead of solving a full optimization problem, the algorithm takes one step towards the optimum at each time step, and eventually converges to the optimum. This approach is motivated by \cite{mukherjee2024efficient}, which studies the Best Arm Identification (BAI) problem, where the arms follow a scalar parametric distribution, and the objective is to identify the arm with the highest mean value.

The main contributions of this paper are as follows. 1) We study the BC problem under a more general framework where the arms follow any vector parametric distribution (under mild regularity conditions), where the arms in a cluster may have different distributions. 2) We propose an Efficient Bandit Clustering algorithm (EBC), which is $\delta$-PC, i.e., the error probability at the time of stopping is at most $\delta$, and asymptotically optimal, i.e., the expected sample complexity grows at the same rate as the lower bound in the low probability regime. 3) We highlight the computational efficiency of EBC by comparing its per-sample run time with that of existing algorithms. 4) We validate the asymptotic optimality of EBC through simulations on synthetic datasets and demonstrate its performance gain over existing approaches on both synthetic and real-world datasets. 5) We propose a heuristic variant of EBC, called EBC-H, which performs slightly better than EBC in both sample and computational complexity in simulations.

Proofs of our theoretical analysis are relegated to the technical appendix.

\section{Problem Setup} \label{sec: problem setup}

\def\xsize{0.18}
\begin{figure}[b]
\vspace{-2.5mm}
    \centering
    \begin{subfigure}[b]{0.11\textwidth}
        \centering
        \begin{tikzpicture}[scale=0.25]
            \foreach \x/\y in {-1/-2, -1/-1, 1/1, 2/2, 3/-3, 3.5/-3} {
                \filldraw[fill=yellow, fill opacity=0.5, draw=black, thick] (\x,\y) ellipse (0.35 and 0.35);
            }
            \draw[gray!70] (-4,-4) grid (4,4);
            \draw[thick,->] (-4,0)--(4.2,0);
            \draw[thick,->] (0,-4)--(0,4.2);

            \node[font=\small, right] at (-3,-2) {1};
            \node[font=\small, right] at (-3,-1) {2};
            \node[font=\small, right] at (-0.6,0.7) {3};
            \node[font=\small, right] at (2.1,2.4) {4};
            \node[font=\small, right] at (1, -3) {5};
            \node[font=\small, right] at (2.8,-1.8) {6};
            
            \foreach \x/\y in {-1/-2, -1/-1, 1/1, 2/2, 3/-3, 3.5/-3} {
                \draw[red, line width=1pt] (\x-\xsize,\y-\xsize)--(\x+\xsize,\y+\xsize);
                \draw[red, line width=1pt] (\x-\xsize,\y+\xsize)--(\x+\xsize,\y-\xsize);
            }
        \end{tikzpicture}
        \caption{}
        \label{a}
    \end{subfigure}
    \begin{subfigure}[b]{0.11\textwidth}
        \centering
        \begin{tikzpicture}[scale=0.25]
            \foreach \x/\y in {-1/-2, -1/-1, 1/1, 2/2} {
                \filldraw[fill=yellow, fill opacity=0.5, draw=black, thick] (\x,\y) ellipse (0.35 and 0.35);
            }
            \filldraw[fill=yellow, fill opacity=0.5, draw=black, thick] (3.25,-3) ellipse (0.8 and 0.6);
            \draw[gray!70] (-4,-4) grid (4,4);
            \draw[thick,->] (-4,0)--(4.2,0);
            \draw[thick,->] (0,-4)--(0,4.2);
            \node[font=\small, right] at (-3,-2) {1};
            \node[font=\small, right] at (-3,-1) {2};
            \node[font=\small, right] at (-0.6,0.7) {3};
            \node[font=\small, right] at (2.1,2.4) {4};
            \node[font=\small, right] at (1, -3) {5};
            \node[font=\small, right] at (2.8,-1.8) {6};
            \foreach \x/\y in {-1/-2, -1/-1, 1/1, 2/2, 3/-3, 3.5/-3} {
                \draw[red, line width=1pt] (\x-\xsize,\y-\xsize)--(\x+\xsize,\y+\xsize);
                \draw[red, line width=1pt] (\x-\xsize,\y+\xsize)--(\x+\xsize,\y-\xsize);
            }
        \end{tikzpicture}
        \caption{}
        \label{b}
    \end{subfigure}
    \begin{subfigure}[b]{0.11\textwidth}
        \centering
        \begin{tikzpicture}[scale=0.25]
            \foreach \x/\y in {1/1, 2/2} {
                \filldraw[fill=yellow, fill opacity=0.5, draw=black, thick] (\x,\y) ellipse (0.35 and 0.35);
            }
            \filldraw[fill=yellow, fill opacity=0.5, draw=black, thick] (-1,-1.5) ellipse (0.6 and 1);
            \filldraw[fill=yellow, fill opacity=0.5, draw=black, thick] (3.25,-3) ellipse (0.8 and 0.6);
            \draw[gray!70] (-4,-4) grid (4,4);
            \draw[thick,->] (-4,0)--(4.2,0);
            \draw[thick,->] (0,-4)--(0,4.2);
            \node[font=\small, right] at (-3,-2) {1};
            \node[font=\small, right] at (-3,-1) {2};
            \node[font=\small, right] at (-0.6,0.7) {3};
            \node[font=\small, right] at (2.1,2.4) {4};
            \node[font=\small, right] at (1, -3) {5};
            \node[font=\small, right] at (2.8,-1.8) {6};
            \foreach \x/\y in {-1/-2, -1/-1, 1/1, 2/2, 3/-3, 3.5/-3} {
                \draw[red, line width=1pt] (\x-\xsize,\y-\xsize)--(\x+\xsize,\y+\xsize);
                \draw[red, line width=1pt] (\x-\xsize,\y+\xsize)--(\x+\xsize,\y-\xsize);
            }
        \end{tikzpicture}
        \caption{}
        \label{c}
    \end{subfigure}
    \begin{subfigure}[b]{0.11\textwidth}
        \centering
        \begin{tikzpicture}[scale=0.25]
            \filldraw[fill=yellow, fill opacity=0.5, draw=black, thick] (1.5,1.5) ellipse (1 and 1);
            \filldraw[fill=yellow, fill opacity=0.5, draw=black, thick] (-1,-1.5) ellipse (0.6 and 1);
            \filldraw[fill=yellow, fill opacity=0.5, draw=black, thick] (3.25,-3) ellipse (0.8 and 0.6);
            \draw[gray!70] (-4,-4) grid (4,4);
            \draw[thick,->] (-4,0)--(4.2,0);
            \draw[thick,->] (0,-4)--(0,4.2);
            \node[font=\small, right] at (-3,-2) {1};
            \node[font=\small, right] at (-3,-1) {2};
            \node[font=\small, right] at (-0.6,0.7) {3};
            \node[font=\small, right] at (2.1,2.4) {4};
            \node[font=\small, right] at (1, -3) {5};
            \node[font=\small, right] at (2.8,-1.8) {6};
            
            \foreach \x/\y in {-1/-2, -1/-1, 1/1, 2/2, 3/-3, 3.5/-3} {
                \draw[red, line width=1pt] (\x-\xsize,\y-\xsize)--(\x+\xsize,\y+\xsize);
                \draw[red, line width=1pt] (\x-\xsize,\y+\xsize)--(\x+\xsize,\y-\xsize);
            }
        \end{tikzpicture}
        \caption{}
        \label{d}
    \end{subfigure}
    \caption{
    Consider a BC problem with $d=2$, $K=3$, $M=6$, and the parameter vectors $\btheta = \begin{bmatrix}
        -1 & -1 & 1 & 2 & 3 & 3.5 \\
        -1 & -2 & 1 & 2 & -3 & -3
    \end{bmatrix}$. SLINK initially assumes each point as a cluster, as shown in Fig.~\ref {a}. Then, the two closest clusters are subsequently merged until the number of clusters reaches $K = 3$, as shown in Figs.~\ref{b}, \ref{c}, and~\ref{d}.
    }
    \label{fig: SLINKdescription}
\end{figure}
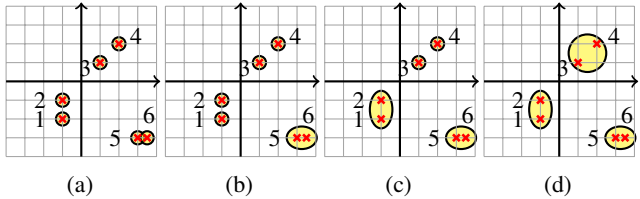

We consider a collection of $M$ data sequences (arms) $\left\{ \bX^{(m)}: m \in [M] \right\}$, where $\bX^{(m)}$ is the $m^{\text{th}}$ data sequence. Each data sequence is a sequence of independent and identically distributed (i.i.d.) samples, i.e., $\bX^{(m)}=\left\{\bX_s^{(m)}: s\in \mathbb{N}\right\}$, $\forall m \in [M]$, where $\bX_s^{(m)}$ is a sample from $m^{\text{th}}$ arm at time $s$. Each sample $\bX^{(m)}_s$ is independently generated from a parametric distribution $P_m(\cdot \mid \btheta_m)$ with a $d$-dimensional parameter vector $\btheta_m \in \Theta$.
The parameter space $\Theta \subset \mathbb{R}^d$ is compact, which is known to the learner.
Let $\btheta \in \Theta^M$ denotes the collection of parameters of the $M$ arms, i.e., $\btheta \coloneqq [\btheta_1, \dots, \btheta_M]$, which we refer to as a problem instance.
 $M$ arms form $K$ clusters based on the Single Linkage Clustering algorithm (SLINK) \cite{rohlf198212} applied on the collection of parameters $\btheta$. We assume that the number of clusters $K$ is known.
SLINK is explained in Fig.~\ref{fig: SLINKdescription} through an illustrative example. We use the cluster index vector $\bc=[c_1, \ldots, c_M]\in [K]^M$ to denote the clustering of the arms. Here, $c_m=k$ denotes that the $m^{\text{th}}$ arm belongs to the $k^{\text{th}}$ cluster. 
Note that the cluster index vector for the given grouping of arms does not need to be unique. For the example considered in Fig.~\ref{fig: SLINKdescription}, both $\bc^{(1)}=[1, 1, 2, 2, 3, 3]$ and $\bc^{(2)}=[2, 2, 1, 1, 3, 3]$ are the valid cluster index vectors. Hence, they are considered equivalent. Formally, for any two cluster index vectors $c^{(1)}$ and $c^{(2)}$, if there exists a permutation $\sigma$ on $[K]$ such that $c^{(1)} = \sigma(c^{(2)})$, then $c^{(1)}$ and $c^{(2)}$ are said to be equivalent, denoted by $c^{(1)} \sim c^{(2)}$.
Let $\mathcal{C}:\Theta^M\rightarrow[K]^M$ be a relation that takes the collection parameters $\btheta$ as input, applies SLINK, and outputs the cluster index vector $\bc$, i.e., $\bc \sim \mathcal{C}(\btheta)$. 

The collection of parameters of $M$ arms $\btheta$ is unknown. We sequentially observe a sample under the bandit framework, i.e., at each time $t$, based on the observed samples from the past $t-1$ times, we can adaptively select an arm $A_{t}\in [M]$ and observe a sample. The objective is to estimate the $K$ clusters among the $M$ arms while observing as few samples as possible, subject to a fixed error probability. The more formal description of the objective is as follows. We use $\mathbb{P}^\pi_{\btheta}$ and $\mathbb{E}_{\btheta}^\pi$ to represent the probability measure and the expectation measure under the problem instance $\btheta$ and the algorithm $\pi$. Let $\tau_\delta(\pi)$ be the stopping time of the algorithm and $\bc_{\tau_\delta}(\pi)$ be the estimated cluster index vector at the stopping time, for some fixed error probability $\delta\in(0, 1)$.
\begin{definition} \label{definition: deltapc}
    An algorithm $\pi$ is said to be $\delta$-Probably Correct ($\delta$-PC), if the algorithm $\pi$ stops in finite time almost surely, i.e., $\mathbb{P}_{\btheta}^{\pi}[\tau_\delta(\pi)<\infty]=1$, and the probability of error is upper bound by $\delta\in (0, 1)$, i.e., $\mathbb{P}_{\btheta}^\pi\left[ \bc_{\tau_\delta}(\pi) \nsim \mathcal{C}(\btheta) \right]\leq \delta$.
\end{definition}
As discussed in Section \ref{sec: intro}, this problem has been studied in \cite{chandran2025online} for Gaussian-distributed arms; however, the designed $\delta-$PC algorithm, ATBOC, is computationally costly. Hence, the objective is to design a computationally efficient $\delta$-PC algorithm $\pi$ (Definition \ref{definition: deltapc}) with the expected sample complexity $\mathbb{E}_{\btheta}^\pi[\tau_\delta(\pi)]$ as low as possible, subject to a fixed error probability $\delta$.

In our work, we consider a collection of probability distributions $\left\{P_m(\cdot\mid\btheta_m):m\in [M]\right\}$ for $M$ arms which satisfy the following assumptions.
\begin{enumerate}
    \item The log-likelihood function $\log P_m(\cdot\mid \brho)$ is twice differentiable and concave in $\brho \in \Theta$.
    \item $\mathbb{E}_{\bX\sim P_m(\cdot\mid\btheta_m)}\left[ \left\| \nabla_{\brho} \log P_m(\bX\mid \brho) \big|_{\brho=\btheta_m} \right\|^3 \right]<\infty$.
    \item Fisher information $\mathcal{I}_m(\brho) = \mathbb{E}_{\bX\sim P_m(\cdot\mid\brho)}\left[\lambda_{\text{max}}\left( -\nabla_{\brho}^2\log P_m(\bX\mid \brho) \right)\right]<\infty$.
    \item $\lambda_{\text{min}}\left( -\nabla_{\brho}^2\log P_m(\cdot\mid \brho) \right)\geq \sigma^2$, for some $\sigma^2>0$.
    \item KL-divergence $D_{\text{KL}}^{(m)}(\brho_1, \brho_2) = \mathbb{E}_{\bX\sim P_m(\cdot\mid\brho_1)}\left[ \log\frac{P_m(X\mid\brho_1)}{P_m(X\mid\brho_2)} \right]$ is uniformly continuous with respect to $\brho_1,\brho_2\in \Theta$, $\forall m \in [M]$.
\end{enumerate}
We use $\lambda_{\text{max}}(\cdot)$ and $\lambda_{\text{min}}(\cdot)$ to denote the maximum and minimum eigen values respectively. 
The efficient BAI studied in \cite{mukherjee2024efficient} considers scalar-parameter distributions for arms, with the parameters being their means. We broaden this class of distributions by considering vector-parameter distributions, satisfying these assumptions, with the parameters not necessarily being their means.
A broad class of distributions satisfy these assumptions. In particular, the vector parameter exponential family of distributions fall within this class.

\section{Lower bound}
In this section, we present a problem instance-dependent and algorithm-independent information-theoretic asymptotic lower bound on the expected stopping time of an algorithm for the BC problem discussed in Section \ref{sec: problem setup}. Define the probability simplex $\mathcal{P}_M = \left\{ \bw\in \mathbb{R}_{+}^M: \sum_{m=1}^Mw_m = 1\right\}$, where $\bw = [w_1, \ldots, w_M]$. We use $\text{KL}(p, q)$ to denote the KL-divergence between the Bernoulli distribution with means $p$ and $q$. We define the alternative space of $\btheta$, denoted by $\text{Alt}(\btheta)$, as the set of all problem instances $\blambda$ that form different clusters than that  of $\btheta$, i.e., $\text{Alt}(\btheta)\coloneqq \left\{ \blambda\in \Theta^M: \mathcal{C}(\btheta)\nsim \mathcal{C}(\blambda) \right\}$, where any $\blambda\in \text{Alt}(\btheta)$ is an alternative problem instance to $\btheta$.  
\begin{theorem} \label{theorem: lowerbound}
    Consider $\delta \in (0, 1)$. For a problem instance $\btheta \in \Theta^M$, any $\delta$-PC algorithm $\pi$ (Definition \ref{definition: deltapc}) satisfies $\mathbb{E}_{\btheta}^\pi\left[ \tau_\delta(\pi) \right] \geq \text{KL}(\delta, 1-\delta)T^{*}(\btheta)$, where
    \begin{equation}
        T^{*}(\btheta) \coloneqq \left[\sup_{\bw \in \mathcal{P}_M} \inf_{\blambda\in \text{Alt}(\btheta)} \sum_{m=1}^M w_m D_{\text{KL}}^{(m)}(\btheta_m, \blambda_m)\right]^{-1}.
    \end{equation}
    Furthermore, $\liminf_{\delta\rightarrow 0}\frac{\mathbb{E}_{\btheta}^{\pi}[\tau_\delta(\pi)]}{\log\left(\frac{1}{\delta}\right)} \geq T^{*}(\btheta)$.
\end{theorem}
We prove Theorem \ref{theorem: lowerbound} by using the Transportation Cost inequality presented in Lemma 1 in \cite{kaufmann2016complexity}.
We use $\psi(\bw, \btheta)$ to denote the inner infimum problem in Theorem \ref{theorem: lowerbound}, i.e., 
$\psi(\bw, \btheta) \coloneqq   \inf_{\blambda\in \text{Alt}(\btheta)} \sum_{m=1}^M w_m D_{\text{KL}}^{(m)}(\btheta_m, \blambda_m).$
\begin{lemma} \label{lemma: psicont}
    The function $\psi(\cdot, \cdot)$ is continuous in $\mathcal{P}_M\times\Theta^M$.
\end{lemma}
We prove Lemma \ref{lemma: psicont} using Assumption 5. From Lemma \ref{lemma: psicont} and the fact that $\mathcal{P}_M$ is compact, the $\sup$ in Theorem \ref{theorem: lowerbound} can be replaced with $\max$ and the optimizer is given by 
$S^{*}(\btheta) = \argmax_{\bw \in \mathcal{P}_M} \psi(\bw, \btheta)$. Here, $\bw$ can be interpreted as the arm pull proportions, i.e., the fraction of times each arm is sampled. Hence, $\mathcal{S}^{*}(\btheta)$ is the set of all optimal arm pull proportions that maximize the weighted KL-distance between the true instance and its closest alternative.

\section{Efficient Bandit Clustering Algorithm (EBC) and its performance analysis}
The pseudo code for the proposed Efficient Bandit Clustering algorithm (EBC) is presented in Algorithm~\ref{Algo:EBC}. EBC, a sequential algorithm,  does the following at each time $t$. 
\begin{itemize}
    \item \textbf{Sampling:} Select an arm to observe a sample.
    \item \textbf{Stopping:} Decides whether to stop or continue.
    \item \textbf{Declaration:} Declares the estimated clusters on stopping.
\end{itemize}

\begin{algorithm} 
\caption{\textit{EBC}}\label{Algo:EBC}
\begin{algorithmic}[1]
\STATE \textbf{Input:} $\delta$,  $K$
\STATE \textbf{Initialize:} $t=0$, $N_m(0)=0, \forall  m \in [M]$, $\bw(0)\in \mathcal{P}_M$.

\REPEAT
    \IF{$ \min_{m \in [M]} N_m(t) < \sqrt{\frac{t}{M}}$}
        \STATE $
        A_{t+1} = \argmin_{m \in [M]} N_m(t)$
    \ELSE
        \STATE $\displaystyle A_{t+1} = \argmin_{m \in [M]}  \left[ \frac{N_m(t)}{t} - \overline{w}_m(t) \right] $
    \ENDIF
    \STATE Observe a sample from data sequence $A_{t+1}$
    \STATE $t \leftarrow t+1$, $N_{A_t}(t)$=$N_{A_t}(t-1)+1$. Update $\hat{\theta}_{A_t}(t)$.
    \STATE Compute the gradient $\bg_t$ at $\bw(t-1)$. (Eq. \eqref{eq: grad})
    \STATE Compute $\bw^{'}(t)=\bw(t-1)+\eta \bg_t$ and project it to probability simplex $\bw(t) = \argmax_{\bw\in\mathcal{P}_M}\|\bw-\bw^{'}(t)\|$.
    \STATE Compute $\overline{\bw}(t) = \frac{1}{t}\sum_{s=1}^t\bw(s)$
    \STATE Compute $Z(t) = t \psi\left( \frac{\boldsymbol{N}(t)}{t}, \boldsymbol{\hat{\theta}}(t) \right)$ and $\beta\left( \delta, t \right)$.

\UNTIL{$Z(t) \geq \beta(\delta, t)$} 

\STATE $\boldsymbol{\hat{c}} = \mathcal{C}\left( \boldsymbol{\hat{\theta}}(t) \right)$
\STATE \textbf{Output:} $\boldsymbol{\hat{c}}$
\end{algorithmic}
\end{algorithm}
Now, we discuss each of them in detail. \\
\textbf{Sampling rule:}
Define $\bN(t)\coloneqq [N_1(t), \ldots, N_M(t)]$, where $N_m(t)$ is the number of samples observed from arm $m$ till time $t$. We call $\bw(t)$ as the estimate of optimal arm pull proportions $\bw^{*}\in \mathcal{S}^{*} (\btheta)$ at time $t$. We define $\overline{\bw}(t)$ as the average of estimates till time $t$, i.e., $\overline{\bw}(t)\coloneqq \frac{1}{t}\sum_{s=1}^t \bw(s)$.
We use $\hat{\btheta}(t) = \left[\hat{\btheta}_1(t), \dots, \hat{\btheta}_M(t)\right]$, where $\hat{\btheta}_m(t)$ is the maximum likelihood estimate of arm $m$ at time $t$ projected into the known compact space $\Theta$, i.e., 
$\hat{\btheta}_m(t) = \argmax_{\brho \in \Theta}\sum_{s\in[t]:A_s=t}\log P_m\left( \bX_s^{(m)}\mid \brho \right).$

The sampling rule involves two components- Forced Exploration and Gradient Tracking. Forced exploration ensures that each arm is sampled at least on the order of $\sqrt{t}$ (Lines 4 and 5 in Algorithm \ref{Algo:EBC}). 
We randomly initialize the arm pull proportion estimate $\bw(0)$ (Line 2). At each time $t$, we compute the gradient of the inner infimum function $\psi(\cdot, \cdot)$, i.e., 
$
\bg_t = \nabla_{\bw}\psi(\bw, \hat{\btheta}(t)) \big|_{\bw=\bw(t-1)}.
$
From Danskin's theorem \cite{bertsekas1971control}, the above gradient can be computed as 
\begin{equation} \label{eq: grad}
    \bg_t = \left[ D_{\text{KL}}^{(1)}(\hat{\btheta}_1(t), \blambda_1^{*}), \ldots, D_{\text{KL}}^{(M)}(\hat{\btheta}_M(t), \blambda_M^{*})  \right],
\end{equation}
where $\displaystyle \blambda^{*} \in \arginf_{\blambda\in \text{Alt}(\hat{\btheta}(t))} \sum_{m=1}^M w_m D_{\text{KL}}^{(m)}(\hat{\btheta}_m(t), \blambda_m)$.
In Gradient tracking, we use the computed gradient $\bg_t$ to perform a one-step gradient descent update and project it onto the probability simplex, obtaining the new estimate $\bw(t)$ (Line 12). To project it into the probability simplex, we use the algorithm proposed in \cite{chen2011projection}. Then, we find the average of estimates $\overline{\bw}(t)$ (Line 13) and select an arm to track $\overline{\bw}(t)$ (Line 7). 

The forced exploration in the sampling rule ensures the convergence of the estimated parameter $\hat{\btheta}_m(t)$ to the true parameter $\btheta_m$ as presented in Lemma \ref{lemma: parameter converge}.
\begin{lemma} \label{lemma: parameter converge}
    For any $\epsilon_1>0$, there exist a stochastic time $N_{\epsilon_1}^S$ satisfying $\mathbb{E}[N_{\epsilon_1}^S]<\infty$, such that for all $t>N_{\epsilon_1}^S$, $\left\| \hat{\btheta}_m(t)-\btheta_m \right\|<\epsilon_1$, for all $m\in[M]$.
\end{lemma}
We prove Lemma \ref{lemma: parameter converge} by leveraging the results on $r$-quick convergence \cite{tartakovsky2023quick} and Assumption 2.
The above-discussed sampling rule ensures that the empirical arm pull proportions converge to the optimal proportions as presented in Lemma \ref{prop: armpullpropconverge}. 
\begin{lemma} \label{prop: armpullpropconverge}
For any $\epsilon_1>0$, there exist a stochastic time $N_{\epsilon_1}^S$ satisfying $\mathbb{E}[N_{\epsilon_1}^S]<\infty$, such that for all $t>N_{\epsilon_1}^S$, $\left|\frac{N_m(t)}{t}-w_m^{*}\right|<\epsilon_1$, for all $m\in[M]$, for some $\bw^{*}\in\mathcal{S}^{*}(\btheta)$.
\end{lemma}
To prove Lemma \ref{prop: armpullpropconverge}, we first show that the sequence $\{{\overline{\bw}(t) : t \in \mathbb{N}}\}$ converges to an optimal arm-pull proportions $\bw^{*} \in \mathcal{S}^{}(\btheta)$ using Lemmas \ref{lemma: psicont} and \ref{lemma: parameter converge}. We then show that any sampling rule that tracks such a sequence ensures the convergence of the empirical arm-pull proportions to the optimal proportions.
\begin{remark}
    The sampling rule of the existing asymptotically optimal BC algorithms in the literature requires solving the $\sup\inf$ problem in Theorem \ref{theorem: lowerbound} at each time $t$, which is computationally costly. Such a computationally costly sampling rule was to ensure that the arm pull proportions $\frac{\bN(t)}{t}$ converge to the optimal arm pull proportions $\bw^{*}\in\mathcal{S}^{*}(\btheta)$. In our sampling rule, at each time $t$, we compute the gradient of the inner infimum problem. Its computation is equivalent to solving only the $\inf$ problem. This results in a considerable reduction in run time, as evident from Table \ref{Table: complexity} by comparing EBC and ATBOC. Moreover, this low-complexity sampling rule ensures the convergence of empirical proportions to optimal proportions (Lemma \ref{prop: armpullpropconverge}). 
\end{remark}
\noindent\textbf{Stopping rule:}
We use $Z(t)$ and $\beta(\delta,t)$ to denote the test statistic and threshold,
respectively. The algorithm stops when $Z(t) > \beta(\delta,t)$; otherwise, it
continues sampling (Line~15). The test statistic is the generalized likelihood
ratio, i.e.,
\begin{equation}
    Z(t) = \log\frac{\displaystyle \max_{\blambda: \mathcal{C}(\blambda)\sim \mathcal{C}\left(\hat{\btheta}(t)\right)}\prod_{m\in [M]}\prod_{s\in[t]:A_s=m}P_m\left(\bX_s^{(m)}\mid \blambda\right)}{\displaystyle \max_{\blambda: \mathcal{C}(\blambda)\nsim \mathcal{C}\left(\hat{\btheta}(t)\right)}\prod_{m\in [M]}\prod_{s\in[t]:A_s=m}P_m\left(\bX_s^{(m)}\mid \blambda\right)}.
\end{equation}
The above expression on simplification yields $Z(t) = t\psi\left(\frac{\bN(t)}{t}, \hat{\btheta}(t)\right)$. The threshold used in the stopping rule is 
\begin{equation}
\begin{aligned}
    &\beta(\delta, t) = \frac{d}{2}\sum_{m\in[M]}\log(\mathcal{I}_m(\hat{\btheta}_m(t))N_m(t)) +\log\left(\frac{1}{\delta}\right) \\
    &- d\sum_{m\in [M]} W_m^{\epsilon}\left(t\right) + Md\log\left( \frac{\sqrt[d]{|\Theta|}}{\sqrt{2\pi}} \right) + \sum_{m\in[M]}N_m(t)\\
    &\sum_{i\in[d]}\max\left\{ D_{\text{KL}}\left( \hat{\btheta}(t), \hat{\btheta}(t)-\epsilon\be_i \right), D_{\text{KL}}\left( \hat{\btheta}(t), \hat{\btheta}(t)+\epsilon\be_i \right) \right\},
\end{aligned}
\end{equation}
where $W_m^{\epsilon}\left(t\right) = \mathbb{E}_{\hat{\btheta}_m(t)}\left[\log\left(1-2Q\left(\epsilon\sqrt{\lambda_{\text{max}}(V_m(t))}\right)\right)\right]$, with $V_m(t)=-\sum_{m\in[M]}\nabla^2_{\brho}\log P_m(\bX\mid \brho)\big|_{\brho=\hat{\btheta}_m(t)}$.
The choice of this threshold is to ensure that the EBC is a $\delta$-PC.

\noindent \textbf{Declaration rule: }
Upon stopping, EBC outputs the estimated cluster index vector by applying SLINK to $\hat{\btheta}(t)$ (Line 16).

 Theorem \ref{theorem: deltaPC} presents that EBC is $\delta-$PC, and Theorem \ref{theorem: AsymptoticOptimal} presents that EBC is asymptotically optimal.
\begin{theorem} \label{theorem: deltaPC}
    EBC satisfies $\mathbb{P}_{\boldsymbol{\theta}}^{\text{EBC}}\left[ \tau_\delta(\text{EBC}) < \infty \right] = 1$ and $\mathbb{P}_{\boldsymbol{\theta}}^{\text{EBC}}\left[ \mathcal{C}\left( \hat{\boldsymbol{\theta}}\left( \tau_\delta(\text{EBC}) \right) \right) \nsim \mathcal{C}\left( \boldsymbol{\theta} \right) \right] \leq \delta$.
\end{theorem}
\begin{theorem} \label{theorem: AsymptoticOptimal}
    For $\epsilon>0,$ EBC algorithm satisfies $\displaystyle \limsup_{\delta \rightarrow 0} \frac{\mathbb{E}_{\boldsymbol{\theta}}^{\text{EBC}}[\tau_\delta(\text{EBC})]}{\log\left( \frac{1}{\delta} \right)} \leq \frac{1}{\psi(\bw^{*}, \btheta)-\eta(\epsilon)},$ where, $\displaystyle \eta(\epsilon)=\max_{\brho\in\Theta}\max_{m\in[M]}\sum_{i\in[d]}\max\left\{ D_{\text{KL}}(\brho, \brho-\epsilon\be_i), D_{\text{KL}}(\brho, \brho+\epsilon\be_i) \right\}.$  
\end{theorem}
To prove that EBC is $\delta$-PC (Theorem \ref{theorem: deltaPC}), we first derive a martingale-based upper bound on the KL divergence between the estimated $\hat{\btheta}_m(t)$ and true parameter $\hat{\btheta}_m(t)$, and then use this bound, along with Assumptions 1, 3, and 4, and Ville’s inequality to establish the result. To prove the sample complexity result (Theorem \ref{theorem: AsymptoticOptimal}), first we observe the fact that $Z(t)$ grows linearly and the threshold grows logarithmically, and then we apply the Lambert-W based bound in \cite{jedra2020optimal}. From Theorem~\ref{theorem: AsymptoticOptimal}, we observe that by setting $\epsilon$ close to $0$, the asymptotic slope of EBC approaches the lower-bound slope (Theorem~\ref{theorem: lowerbound}); hence, EBC is asymptotically optimal.


\begin{figure}
    \centering
        \begin{tikzpicture}
        \begin{axis}[
            xlabel = {$\log\left(\frac{1}{\delta}\right)$}, 
            ylabel = {\small{Expected Stopping time}},
            xmin =0, xmax=205,
            ymin=800, ymax=14000,
            xtick={0, 40, 80, 120, 160, 200},
            xticklabels={0, 40, 80, 120, 160, 200},
            grid=major,
            ytick = {1000, 2000, 3000, 4000, 5000, 6000, 7000, 8000, 9000, 10000, 11000, 12000, 13000, 14000},
            yticklabels = {0.1, 0.2, 0.3, 0.4, 0.5, 0.6, 0.7, 0.8, 0.9, 1.0, 1.1, 1.2, 1.3, 1.4},
            legend style={at={(0, 1)},opacity=0.8, legend cell align=left,anchor=north west, nodes={scale=0.70}},
            width = 0.75\columnwidth, 
            height = 0.7\columnwidth, 
            tick label style={font=\tiny},
            enlargelimits=false,
        ]

        \addplot[blue, thick, mark=o, mark size=2pt] coordinates {
            (1.0,   1485.375)
            (23.11, 2746.84)
            (45.22, 3899.91)
            (67.33, 5005.465)
            (89.44, 6123.06)
            (111.56,7242.93)
            (133.67,8312.55)
            (155.78,9401.47)
            (177.89,10498.74)
            (200.0, 11586.845)
        };

        \addplot[red, thick, mark=square, mark size=2pt] coordinates {
            (1.0,   1216.68)
            (23.11, 2280.63)
            (45.22, 3230.71)
            (67.33, 4207.62)
            (89.44, 5085.44)
            (111.56,6008.60)
            (133.67,6894.17)
            (155.78,7792.87)
            (177.89,8685.21)
            (200.0, 9568.00)
        };

        \addplot[green!60!black, thick, mark=x] coordinates {
            (1.0,   2784.86)
            (23.11, 3324.36)
            (45.22, 3878.84)
            (67.33, 4412.2)
            (89.44, 4955.39)
            (111.56,5497.23)
            (133.67,6020.11)
            (155.78,6559.17)
            (177.89,7091.86)
            (200.0, 7659.92)
        };

 \addplot[orange, thick, mark=+] coordinates {
            (1.0,   2606.24333333)
            (23.11, 3162.17666667)
            (45.22, 3710.61333333)
            (67.33, 4247.86)
            (89.44, 4796.88333333)
            (111.56,5331.13)
            (133.67,5851.50333333)
            (155.78,6389.19)
            (177.89,6922.37)
            (200.0, 7442.31333333)
        };
        
        \addplot[purple, thick, mark=x] coordinates {
            (1.0,   2662.81)
            (23.11, 3075.77)
            (45.22, 3491.28)
            (67.33, 3938.72)
            (89.44, 4345.54)
            (111.56,4777.47)
            (133.67,5190.26)
            (155.78,5630.82)
            (177.89,6041.88)
            (200.0, 6469.27)
        };

        \addplot[cyan!70!black, thick, mark=+] coordinates {
            (1.0,   2513.18666667)
            (23.11, 2938.50666667)
            (45.22, 3356.3)
            (67.33, 3777.31333333)
            (89.44, 4195.88333333)
            (111.56,4624.47)
            (133.67,5052.00333333)
            (155.78,5491.36)
            (177.89,5908.91333333)
            (200.0, 6335.11333333)
        };

        \addplot[brown, thick, mark=x] coordinates {
            (1.0,   1622.98)
            (23.11, 2070.04)
            (45.22, 2456.98)
            (67.33, 2887.49)
            (89.44, 3295.23)
            (111.56,3710.01)
            (133.67,4137.44)
            (155.78,4575.30)
            (177.89,4984.95)
            (200.0, 5427.61)
        };
        
 \addplot[magenta, thick, mark=+] coordinates {
            (1.0,   1472.53333333)
            (23.11, 1897.96666667)
            (45.22, 2321.00333333)
            (67.33, 2747.00666667)
            (89.44, 3170.80333333)
            (111.56,3597.52)
            (133.67,4029.41333333)
            (155.78,4444.64333333)
            (177.89,4864.74666667)
            (200.0, 5292.83)
        };

        \legend{LUCBBOC,
        ATBOC,
        EBC-$|\Theta|=1000^2; \epsilon=0.1$,
        EBC-H-$|\Theta|=1000^2; \epsilon=0.1$,
        EBC-$|\Theta|=1000^2;\epsilon=0.01$,
        EBC-H-$|\Theta|=1000^2;\epsilon=0.01$,
        EBC-$|\Theta|=10^2; \epsilon=0.01$,
        EBC-H-$|\Theta|=10^2; \epsilon=0.01$}
        \end{axis}
        \end{tikzpicture}
        \vspace{-3mm}
        \caption{Synthetic Dataset 1}
        \label{fig:gaussian}
        \vspace{-3mm}
\end{figure}
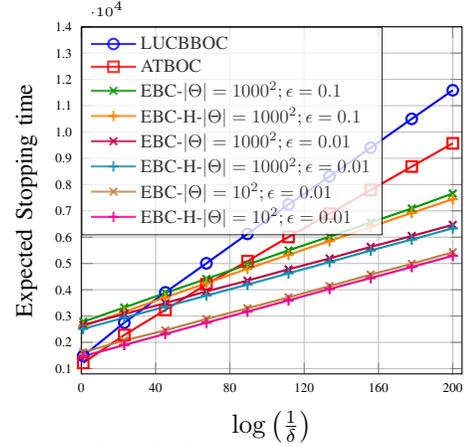
\begin{figure*}[htb]
    \centering
    \begin{minipage}[t]{0.30\textwidth}
        \centering
        \begin{tikzpicture}
\begin{axis}
[
    xlabel = {$\log\left(\frac{1}{\delta}\right)$},
    ylabel = {\small{Expected Stopping Time}},
    xmin = 0,
    xmax = 210,
    ymin = 7000,
    ymax = 35000,
    xtick = {1, 25, 50, 75, 100, 125, 150, 175, 200},
    grid = major,
    legend style={
        at={(0,1)},
        legend cell align=left,
        anchor=north west,
        nodes={scale=0.7}
    },
    width = 1.0\textwidth,
    height = 1.0\textwidth,
    tick label style = {font=\tiny},
    enlargelimits=false,
]


\addplot[green!60!black, thick, mark=x] coordinates {
    (1.0,   9767.15384615)
    (23.11, 12072.47435897)
    (45.22, 14125.79487179)
    (67.33, 16484.83333333)
    (89.44, 18642.03846154)
    (111.56,20851.93589744)
    (133.67,23034.14102564)
    (155.78,25287.06410256)
    (177.89,27375.02564103)
    (200.0, 29549.83333333)
};

\addplot[orange, thick, mark=+] coordinates {
    (1.0,   8715.48333333)
    (23.11, 10952.89666667)
    (45.22, 13136.85333333)
    (67.33, 15322.96333333)
    (89.44, 17537.52)
    (111.56,19722.05333333)
    (133.67,21921.79)
    (155.78,24079.16333333)
    (177.89,26268.14)
    (200.0, 28466.79666667)
};

\addplot[brown, thick, mark=x] coordinates {
    (1.0,   7821.23529412)
    (23.11, 10133.94117647)
    (45.22, 12381.61764706)
    (67.33, 14693.91176471)
    (89.44, 16906.5)
    (111.56,19129.85294118)
    (133.67,21206.5)
    (155.78,23363.97058824)
    (177.89,25535.29411765)
    (200.0, 27753.38235294)
};

\addplot[magenta, thick, mark=+] coordinates {
    (1.0,   6860.8)
    (23.11, 9017.32)
    (45.22, 11207.90666667)
    (67.33, 13360.94666667)
    (89.44, 15577.44)
    (111.56,17717.53333333)
    (133.67,19965.56)
    (155.78,22170.25333333)
    (177.89,24325.28)
    (200.0, 26523.01333333)
};

\legend{ EBC-$|\Theta|=1000$,
        EBC-H-$|\Theta|=1000$, 
        EBC-$|\Theta|=50$,
        EBC-H-$|\Theta|=50$}

\end{axis}
\end{tikzpicture}
        \vspace{-3mm}
        \caption{Synthetic Dataset 2}
        \label{fig:bernoulli}
    \end{minipage}\hfill
    \begin{minipage}[t]{0.30\linewidth}
    \centering
        \begin{tikzpicture}
        \begin{axis}[
    xlabel = {$\log\left(\frac{1}{P_e}\right)$}, 
    ylabel = {Expected Stopping time},
    xmin =0.5, xmax=2.3,
    ymin=300, ymax=2100,
    grid=major,
    legend style={at={(0,1)}, anchor=north west, nodes={scale=0.7}},
    width=\textwidth,
    height=\textwidth,
    tick label style={font=\tiny}
]


\addplot[mark=o,blue,thick] coordinates{
    (0.692, 501)
    (0.885, 701)
    (1.118, 1001)
    (1.372, 1501)
    (1.680, 2001)
};

\addplot[mark=square,red,thick] coordinates{
    (1.091, 673.105)
    (1.239, 907.458)
    (1.420, 1215.395)
    (1.621, 1546.592)
    (1.876, 1988.345)
};

\addplot[mark=x,green!60!black,thick] coordinates{
    (0.87227385, 461.088)
    (1.0441241, 617.6)
    (1.30933332, 841.814)
    (1.46101791, 1027.296)
    (1.59948758, 1237.992)
    (1.73727128, 1437.054)
    ( 2.08747371, 1886.646)
};

 \addplot[mark=+,orange,thick] coordinates{
    (0.89248009, 422.60040161)
    (1.01210303, 493.53815261)
    (1.35078769, 878.38353414)
    (1.52846885, 1048.03413655)
    (1.79175945, 1397.79518072)
    (2.05171696, 1709.8313253)
};
\legend{FSS, RR, EBC,  EBC-H}

\end{axis}
        \end{tikzpicture}
        \vspace{-3mm}
        \caption{Real-world Dataset 1.}
        \label{fig:taxi}
    \end{minipage} \hfill
    \begin{minipage}[t]{0.30\textwidth}
        \centering
        \begin{tikzpicture}
        \begin{axis}[
    xlabel = {$\log\left(\frac{1}{P_e}\right)$}, 
    ylabel = {Expected Stopping time},
    xmin =0.2, xmax=5,
    ymin=5, ymax=75,
    xtick = {1, 2, 3, 4, 5},
    grid=major,
    legend style={at={(0,1)}, anchor=north west, nodes={scale=0.7}},
    width=\textwidth,
    height=\textwidth,
    tick label style={font=\tiny}
]


\addplot[mark=o,blue,thick] coordinates{
    (0.65200524, 10)
    (1.00239343, 15)
    (1.24132859, 20)
    (1.55589715, 25)
    (1.89047544, 30)
    (2.19822508, 35)
    (2.95651156, 50)
    (3.57555077, 60)
    (4.34280592, 70)
};

\addplot[mark=square,red,thick] coordinates{
    (0.93394567, 13.452)
    (1.73727128, 23.948)
    (2.52572864, 34.77)
    (4.26869795, 56.42)
};

\addplot[mark=x,green!60!black,thick] coordinates{
    (0.92130327, 13.832)
    (1.74869998, 23.584)
    (2.55104645, 32.062)
    (3.17008566, 38.844)
    (4.60517019, 50.432)
};

\addplot[mark=+,orange,thick] coordinates{
    (0.88188931, 11.711)
    (1.63475572, 18.002)
    (2.43041846, 23.684)
    (3.19418321, 28.501)
    (4.13516656, 34.861)
};
\legend{FSS, RR, EBC, EBC-H}

\end{axis}
        \end{tikzpicture}
        \vspace{-3mm}
        \caption{Real-world Dataset 2.}
        \label{fig:movielens}
    \end{minipage}
\end{figure*}


\begin{remark} 
    Unlike in our work, the optimal arm pull proportion set $\mathcal{S}^{*}(\btheta)$ is proved to be a singleton in \cite{mukherjee2024efficient} for the BAI problem. 
    Therefore, although the BC problem differs fundamentally from the BAI problem, the theory developed in our work can be used to extend \cite{mukherjee2024efficient} to the multidimensional version of the efficient BAI, where there is no evidence to say $\mathcal{S}^{*}(\btheta)$ is a singleton \cite{jedra2020optimal}. 
\end{remark}
\section{Efficient Bandit Clustering - Heuristic}
We propose a heuristic version of EBC called Efficient Bandit Clustering - Heuristic (EBC-H), which further simplifies the sampling rule, and arm selection is based on the quantities computed in the stopping rule. In EBC-H, we replace Lines 7,11, 12, and 13 in Algorithm \ref{Algo:EBC} with the following sampling rule.
\begin{equation} \label{eq: EBC-Hsamp}
    A_{t+1} = \argmax_{m \in [M]} D_{\text{KL}}^{(m)}\left( \hat{\btheta}_m(t), \blambda_m^{*}\right),
\end{equation}
where, $\displaystyle \blambda^{*}\in \arginf_{\blambda\in\text{Alt}\left(\hat{\btheta}(t)\right)} \sum_{m=1}^M\frac{N_m(t)}{t}D_{\text{KL}}^{(m)}\left( \hat{\btheta}_m(t), \blambda_m\right).$ Note that the quantities $D_{\text{KL}}^{(m)}\left( \hat{\btheta}_m(t), \blambda_m^{*}\right)$, for all $m\in[M]$ will be computed as a part of $Z(t)$ computation in the stopping rule. The idea behind this sampling rule is to compute the gradient of the infimum function $\psi(\bw, \hat{\btheta}(t))$ at $\bw = \frac{\bN(t)}{t}$, and select the arm along which the component of the gradient is maximum as in \eqref{eq: EBC-Hsamp}. 
EBC-H is a $\delta$-PC algorithm as it uses the same stopping rule as in EBC (Algorithm \ref{Algo:EBC}). Sample complexity guarantee of EBC-H is not analyzed in this work. However, from simulation results, we observe that EBC-H follows the asymptotic sample complexity slope guarantees of EBC (Table \ref{Table: slope}) for the examples considered in section \ref{sec: simulation}. Moreover, through simulations, we observe that EBC-H shows better performance than EBC in terms of sample complexity and runtime (Table \ref{Table: complexity}).

\section{Simulation Results} \label{sec: simulation}
We compare the asymptotic performance of our proposed algorithms, EBC and EBC-H, with ATBOC and LUCBBOC \cite{chandran2025online}. To highlight the role of the sampling and stopping rules in the proposed algorithms, we also compare them with the Round Robin algorithm (RR), which uses round-robin sampling, and the Fixed Sample Size algorithm (FSS).
Recall that $\delta$ is an input to the algorithm and serves as an upper bound on the empirical error probability $P_e$. In figures comparing asymptotic performance, we plot results versus $\delta$. In figures comparing real-world performance, we plot results versus $P_e$.

\noindent \textbf{Synthetic Dataset 1:}
Consider the clustering problem in Fig.~\ref{fig: SLINKdescription}, with arms following multivariate Gaussian distributions.
In Fig.~\ref{fig:gaussian}, we observe that EBC and EBC-H show better performance in the asymptotic regime due to their lower slope. We observe that setting a lesser volume of compact space from $\Theta=[-500, 500]^2$ to $[-5, 5]^2$ improves the performance. Table \ref{Table: slope} shows the slope comparison of various algorithms.

\begin{table}[h!]
\centering
\begin{tabular}{|c|c|c|}
\hline
  & Numerical & Theoretical \\ \hline
  Lower bound & - & 20 \\ \hline
  EBC-H-$\epsilon=0.01$ & 19.29 & - \\ \hline
  EBC-$\epsilon=0.01$ & 19.35 & 20 \\ \hline 
  EBC-H-$\epsilon=0.1$ & 23.99 & - \\ \hline
  EBC-$\epsilon=0.1$ & 24.65 & 25 \\ \hline
  ATBOC & 40.31 & 40 \\ \hline
  LUCBBOC & 49.38 & - \\ \hline
\end{tabular}
\caption{Sample complexity slopes for Figure \ref{fig:gaussian}.}
\label{Table: slope}
\end{table}
\begin{table}[h!]
\centering
\begin{tabular}{|c|c|c|c|c|}
\hline
Number of Arms ($M$) & 4 & 8 & 14 & 20 \\ \hline
EBC-H & 0.44 & 0.64 & 1.02 & 1.90 \\ \hline
EBC & 0.62 & 1.11 & 1.81 & 3.52 \\ \hline
LUCBBOC & 0.69 & 1.18 & 2.05 & 3.65 \\ \hline
ATBOC & 21.1 & 91.2 & 348 & 1133 \\ \hline
\end{tabular}
\caption{Average per-sample run time (in ms).}
\label{Table: complexity}
\end{table}


\noindent \textbf{Synthetic Dataset 2:}
Consider $M=6$ arms following exponential distributions with parameters $1, 1.1, 9, 9.1, 20, 20.1$. Figure \ref{fig:bernoulli} shows the performance of EBC and EBC-H on setting $\epsilon=0.001$ and different compact spaces $\Theta=[0.1, 1000]$ and $\Theta=[0.1, 50]$. The slope of EBC ($98.6$) and EBC-H ($98.7$) matches the theoretical lower bound slope ($98$).

\noindent \textbf{Real-world Dataset 1:}
We consider the New York City TLC Trip Record Data of Manhattan City for January 2025 \cite{nyc_tlc_trip_data}. We have $6$ regions in the city of Manhattan. We observe the inter-booking times of taxis in each of these regions. The objective is to group these six regions into three clusters: Quiet, Moderate, and Busy. 
Fig.~\ref{fig:taxi} shows the performance of various algorithms. 
We observe that EBC and EBC-H exhibit the best performance, followed by RR and then FSS.

\noindent \textbf{Real-world Dataset 2:}
We consider the MovieLens dataset, which comprises user ratings for various movies and genres \cite{cantador2011second}. We consider $M=6$ users and group them into $K=4$ clusters. For clustering, we observe users’ ratings for the Comedy and Drama genres; hence, at each round, the algorithm observes $d=2$ dimensional samples. Fig.~\ref{fig:movielens} shows the performance of the various algorithms. We observe that EBC outperforms RR and FSS, and the heuristic variant, EBC-H, further improves upon the performance of EBC.

\noindent \textbf{Computational Complexity comparison}\\
We fix the number of arms per cluster to 2, so that $M = 2K$. The arms follow a $d=5$ dimensional multivariate Gaussian distribution with identity covariance, and the means along each dimension are given by $[0, 0.5, 10, 10.5, \dots, 10(K-1), 10(K-1)+0.5]$. Table \ref{Table: complexity} shows the per-sample run time of the various algorithms. We observe that EBC and EBC-H are more efficient than ATBOC, with computational times comparable to those of the suboptimal alternative, LUCBBOC. Moreover, EBC-H is slightly faster than EBC.
\section{Conclusion}
We proposed efficient BC algorithms EBC and its heuristic variant EBC-H, both of which are $\delta$-PC, for clustering the arms that follow any vector parametric distributions satisfying mild regularity conditions. 
EBC is proven to be asymptotically optimal, and its asymptotic optimality is validated through simulations on synthetic datasets.
Simulation results indicate that EBC-H matches the asymptotic lower bound and exhibits slightly better performance than EBC.
Simulation results on both synthetic and real-world datasets demonstrate that both EBC and EBC-H outperform existing methods in the literature in terms of both sample complexity and runtime. 

\bibliographystyle{IEEEtran}
\bibliography{ref}

\input{appendix}
\end{document}

%% file: appendix.tex
\newpage
\onecolumn
\appendices

\begin{center}
    \Large \bfseries Technical Appendix
\end{center}

When the context $\boldsymbol{\theta}$ and $\pi$ is clear, we represent $\mathbb{E}_{\boldsymbol{\theta}}^{\pi}[\cdot]$ and $\mathbb{P}_{\boldsymbol{\theta}}^{\pi}[\cdot]$ informally as $\mathbb{E}[\cdot]$ and $\mathbb{P}[\cdot]$ respectively. Here $\|\boldsymbol{x}\|$ represents the Euclidean vector norm when $\boldsymbol{x}$ is a vector and represents the Frobenius norm when $\boldsymbol{x}$ is a matrix. For any vector $\boldsymbol{v}$, we use $(\boldsymbol{v})_i$ to indicate the $i^{th}$ coordinate of the vector $\boldsymbol{v}$. We use superscript $S$ in time $N^S\in \mathbb{N}$ to indicate that it is a stochastic time.

\section{Proof of Theorem 1} \label{appsec: lowerbound}
\begin{proof}
The proof follows similar lines to \cite{chandran2025online}. For the sake of completion, we provide the proof here.
For our problem instance $\btheta$, consider an arbitrary instance $\blambda \in \text{Alt}(\boldsymbol{\theta})$.  By applying the transportation inequality in Lemma 1 of \cite{kaufmann2016complexity}, we have 
\begin{equation} \label{eqn:transportaion_lemma}
    \sum_{m=1}^M \mathbb{E}\left[ N_m(\tau_\delta) \right]D_{KL}(\btheta_m, \blambda_m) \geq \text{KL}\left( 
    \delta, 1-\delta \right).
\end{equation}
Since \eqref{eqn:transportaion_lemma} holds for all $\blambda \in \text{Alt}(\btheta)$, we have 
\begin{equation*}
     \inf_{\blambda \in \text{Alt}(\btheta)} \sum_{m=1}^M \mathbb{E}\left[ N_m(\tau_\delta) \right] D_{\text{KL}}(\btheta_m, \blambda_m)  \geq \text{KL}\left( 
    \delta, 1-\delta \right).
\end{equation*}
Therefore,
\begin{align}
    \mathbb{E}\left[\tau_\delta \right]  \inf_{\blambda \in \text{Alt}(\boldsymbol{\theta})} \sum_{m=1}^M 
    \frac{\mathbb{E}\left[ N_m(\tau_\delta) \right]}{\mathbb{E}\left[\tau_\delta\right]}D_{\text{KL}}(\btheta_m, \blambda_m)
    \geq \text{KL}\left( \delta, 1-\delta \right).
\end{align}

    Since $\left[ \mathbb{E}\left[ N_1(\tau_\delta) \right], \dots, \mathbb{E}\left[ N_M(\tau_\delta) \right] \right]^T/\mathbb{E}\left[ \tau_\delta \right]$ forms a probability distribution in $\mathcal{P}_M$, we have
    \begin{align}
    \mathbb{E}\left[\tau_\delta \right] \sup_{\bw \in \mathcal{P}_M} \inf_{\boldsymbol{\lambda} \in \text{Alt}(\boldsymbol{\theta})} 
    \sum_{m=1}^M w_mD_{\text{KL}}(\btheta_m, \blambda_m)
    \geq \text{KL}\left( \delta, 1-\delta \right).
\end{align}

\begin{equation}
    \begin{aligned}
       \therefore \mathbb{E}\left[\tau_\delta\right]  \geq \text{KL}\left( \delta, 1-\delta \right) T^*(\boldsymbol{\theta}).
    \end{aligned}
\end{equation}

\end{proof}

To prove the main theoretical guarantees of the proposed EBC, it requires to prove the following. 
\begin{enumerate}
    \item The inner infimum problem $\psi(\cdot, \cdot)$ is continuous in its domain (Appendix \ref{appsec: lemma1}).
    \item The parameter estimate $\hat{\btheta}(t)$ converges to the true parameter (Appendix \ref{appsec: lemma2}).
    \item The empirical arm pull proportions $\frac{\bN(t)}{t}$ converges to the optimal arm pull proportions $\bw^{*}\in\mathcal{S}^{*}(\btheta)$ (Appendix \ref{appsec: lemma3}).
\end{enumerate}
First, we prove these results in the subsequent sections.

\section{Proof of Lemma 1} \label{appsec: lemma1}
\begin{proof}
Let us consider a point $\left( \bw, \btheta \right) \in \mathcal{P}_M\times \Theta^{M}$. \\
Let $\left(\bw(n), \btheta(n)\right)$ be any arbitrary sequence in $\mathcal{P}_M\times \Theta^{M}$ that converges to $\left( \bw, \btheta \right)$.\\
Now to show $\psi$ is continuous at $\left( \bw, \btheta \right)$, we need to show that sequence $\psi\left( \bw(n), \btheta(n) \right)$ converges to $\psi\left( \bw, \btheta \right)$, i.e., $\lim_{n \rightarrow \infty} \psi\left( \bw(n), \btheta(n) \right) = \psi\left( \bw, \btheta \right)$.\\
Equivalently, we need to show that $\limsup_{n \rightarrow \infty} \psi\left( \bw(n), \btheta(n) \right) \leq \psi\left( \bw, \btheta \right)$ and $\liminf_{n \rightarrow \infty} \psi\left( \bw(n), \btheta(n) \right) \geq \psi\left( \bw, \btheta \right)$.\\
We define the set of all problem instances which has the same clusters as that of $\btheta$ as $\text{CAlt}(\btheta) \coloneqq \left\{ \blambda \in \Theta^d: \mathcal{C}(\blambda)\sim \mathcal{C}(\btheta)\right\}$, where "CAlt" refers to the complement of the alternative space. It can be verified that $\text{CAlt}(\btheta)$ can be represented as $\text{CAlt}(\btheta) = \Lambda\cap\Theta^{M}$, where $\Lambda$ is an open set in $\mathbb{R}^{d\times M}$ and $\Theta$ is a compact parameter space in $\mathbb{R}^d$. 
Clearly, $\btheta \in \text{CAlt}(\btheta)$ and hence $\btheta \in \Lambda$. Since $\Lambda$ is a open set, , we can say that there exist $\epsilon > 0$ such that $\mathcal{B}\left( \btheta, \epsilon \right)\subset \Lambda$, where $\mathcal{B}\left( \btheta, \epsilon \right)\coloneqq \left\{ 
\blambda \in \mathbb{R}^{d\times M} : \|\blambda-\btheta\|<\epsilon \right\}$. \\
Since $\btheta(n)$ converges to $\btheta$,  $\exists N_0 \in \mathbb{N}$ such that for all $n\geq N_0$, we have $\btheta(n) \in \mathcal{B}\left( \btheta, \epsilon \right)$. 
Hence for all $n\geq N_0$, $\btheta(n) \in \Lambda$. Also, the sequence $\btheta(n)$ is defined in $\Theta^M$. Hence, for all $n\geq N_0$, $\btheta(n) \in \text{CAlt}(\btheta)$.\\
Hence, for all $n\geq N_0$, $\left( \bw(n), \btheta(n) \right) \subset \mathcal{P}_M\times \text{CAlt}(\btheta)$.
    Note that, in $\psi(\bw, \btheta)$, infimum is taken over the alternative space $\text{Alt}(\btheta)$. Then there  exists a $\blambda \in \text{Alt}(\btheta)$ such that
    \begin{equation} \label{eq:13}
        \sum_{m = 1}^M w_m d_{\text{KL}}\left(\btheta_m, \blambda_m\right) \leq \psi(\bw, \btheta) + \epsilon.
    \end{equation}
    We have the sequence $\left( \bw(n), \btheta(n) \right)$ converges to a point $\left( \bw, \btheta \right)$ as $n \rightarrow \infty$. Also, we assume that the KL-divergence is uniformly continuous. Hence for $\epsilon$, there exist $N_1$ such that for all $n \geq N_1$ we have the following.
    \begin{equation} \label{eq:14}
        \bw(n) \leq \bw + \epsilon\boldsymbol{1}, \text{ where } \boldsymbol{1} \text{ is the vector of all $1$'s and }
    \end{equation}
    \begin{equation} \label{eq:15}
        d_{\text{KL}}\left(\btheta_m(n), \blambda_m\right) \leq d_{\text{KL}}\left(\btheta_m, \blambda_m\right) + \epsilon \text{ for all } m \in [M].
    \end{equation}
    Note that for $n \geq N_0$, we have $\btheta(n)\in \text{CAlt}(\btheta)$ and hence $\blambda$ considered above also lies in the alternative space of $\btheta(n)$ and hence we have, 
    \begin{equation}
        \psi\left( \bw(n), \btheta(n) \right) \leq \sum_{m = 1}^M w_m(n) d_{\text{KL}}\left(\btheta_m(n), \blambda_m\right)
    \end{equation}
    Now by using equations \eqref{eq:13}, \eqref{eq:14} and \eqref{eq:15} in the above equation, we get the following.
    \begin{equation}
    \begin{split}
        \psi\left( \bw(n), \btheta(n) \right) 
        &\leq \sum_{m = 1}^M (w_m + \epsilon) \left[d_{\text{KL}}\left(\btheta_m, \blambda_m\right) + \epsilon \right] \ \ \text{from \eqref{eq:14} and \eqref{eq:15}} \\ 
        &\leq \sum_{m = 1}^M  w_m d_{\text{KL}}\left(\btheta_m, \blambda_m\right) + \epsilon\left[ 1 + \sum_{m=1}^M d_{\text{KL}}\left(\btheta_m, \blambda_m\right) + \epsilon M \right] \\ 
        &\leq \psi(\bw, \btheta) + \epsilon\left[ 2 + \sum_{m=1}^M d_{\text{KL}}\left(\btheta_m, \blambda_m\right) + \epsilon M \right] \ \ \text{from \eqref{eq:13}}.
    \end{split}
    \end{equation}
    From the assumption that the KL-divergence is uniformly continuous and the the parametric space is compact, $\sum_{m=1}^M d_{\text{KL}}\left(\btheta_m, \blambda_m\right)$ is finite for any $\blambda \in \text{Alt}(\btheta)$. Also $\epsilon$ is arbitrary. Hence, on letting $\epsilon$ tends to $0$, we get for $n\geq \max\{N_0, N_1\}$, 
    \begin{equation}
        \psi\left( \bw(n), \btheta(n) \right) \leq \psi(\bw, \btheta).
    \end{equation}
    Finally on taking $\limsup$ on the above equation, we get, 
    \begin{equation} \label{eq:17}
        \limsup_{n \rightarrow \infty} \psi\left( \bw(n), \btheta(n) \right) \leq \psi\left( \bw, \btheta \right).
    \end{equation}

    For a given $\epsilon$, for each $n\geq N_0$, there exist $\blambda(n) \in \text{Alt}(\btheta(n))$ (Note: for $n\geq N_0$, Alt$\left(\btheta(n)\right)$ = Alt$\left(\btheta\right)$) such that
    \begin{equation} \label{eq:18}
        \psi\left( \bw(n), \btheta(n) \right) \geq \sum_{m=1}^M w_m(n) d_{\text{KL}}\left(\btheta_m(n), \blambda_m(n)\right) - \epsilon
    \end{equation}
    Note that, for $n\geq N_0$, $\btheta(n)$ lies in $\epsilon$ neighborhood around $\btheta$. Hence, $\blambda(n)$ is a bounded sequence, otherwise, right side of the above equation becomes $\infty$, which is not possible. Also, we know for any bounded sequence there exist a converging subsequence. Hence, without loss of generality, let us consider $\blambda(n)$ converges to some point $\blambda \in \text{Alt}(\btheta)$. 
    We have the sequence $\left( \bw(n), \btheta(n) \right)$ converges to a point $\left( \bw, \btheta \right)$ as $n \rightarrow \infty$.
    Also, we assume that the KL-divergence is uniformly continuous.Hence for $\epsilon$,  $\exists N_1$ such that for all $n \geq N_1$ we have the following.
    \begin{equation} \label{eq:19}
        \bw(n) \geq \bw - \epsilon \boldsymbol{1} \text{ and }
    \end{equation}
    \begin{equation} \label{eq:20}
        d_{\text{KL}}\left(\btheta_m(n), \blambda_m(n)\right) \geq d_{\text{KL}}\left(\btheta_m, \blambda_m(n)\right) - \epsilon \text{ for all } m \in [M].
    \end{equation}
    Using equations \eqref{eq:18}, \eqref{eq:19} and \eqref{eq:20}, we get
    \begin{equation}
    \begin{split}
        \psi\left( \bw(n), \btheta(n) \right) 
        &\geq \sum_{m = 1}^M w_m d_{\text{KL}}\left(\btheta_m, \blambda_m(n)\right) - \epsilon\left[ 2 + \sum_{m=1}^M w_md_{\text{KL}}\left(\btheta_m, \blambda_m(n)\right) - \epsilon M \right] \\
        &\geq \psi(\bw, \btheta) - \epsilon\left[ 2 + \sum_{m=1}^M w_md_{\text{KL}}\left(\btheta_m, \blambda_m(n)\right) - \epsilon M \right] \ \ \text{($\because \blambda(n) \in \text{Alt}(\btheta)$)}.
    \end{split}
    \end{equation}
    Since $\sum_{m=1}^M d_{\text{KL}}\left(\btheta_m, \blambda_m(n)\right)$ is finite for any $\blambda \in \text{Alt}(\btheta)$ and $\epsilon$ is arbitrary, on letting $\epsilon$ tends to $0$, we get for $n\geq \max\{N_0, N_1\}$, 
    \begin{equation}
        \psi\left( \bw(n), \btheta(n) \right) \geq \psi\left( \bw, \btheta \right).
    \end{equation}
    Finally on taking $\liminf$ on the above equation, we get, 
    \begin{equation} \label{eq:22}
        \liminf_{n \rightarrow \infty} \psi\left( \bw(n), \btheta(n) \right) \geq \psi\left( \bw, \btheta \right).
    \end{equation}
From equations \eqref{eq:17} and \eqref{eq:22}, we have the continuity of $\psi\left(\bw, \btheta\right)$.
\end{proof}

\section{Proof of Lemma 2} \label{appsec: lemma2}
First, we prove that the maximum likelihood estimate (MLE) of the parameter converges $r$-quickly to the true parameter, for $r=2$, under Assumption 2. To prove this, we proceed as follows.
\begin{enumerate}
    \item We present the definition of the $r$-quick convergence (Definition \ref{def: rquick}).
    \item We prove a sufficient condition for the $r$-quick convergence for multi-dimensional samples (Lemma \ref{lemma: rquick-general}).
    \item We prove equivalent conditions for $r$-quick convergence, for the special case where the sequence of random vectors is i.i.d. (Lemma \ref{lemma: rquick-iid}).
    \item Then, we use Assumption 2 to prove the $r$-quick convergence of MLE for $r=2$ (Lemma \ref{lemma: mle-rquick}).
\end{enumerate}
Finally, we use Lemma \ref{lemma: mle-rquick} and the fact that the algorithm uses forced exploration to prove Lemma \ref{lemma: parameter converge}.

\begin{definition} \label{def: rquick}
    (Definition 2 in \cite{tartakovsky2023quick}) Consider $r>0$ and $\epsilon>0$. Let $\left\{\bZ_t\in \mathbb{R}^d: t\in \mathbb{N}\right\}$ be the sequence of random vectors. Then $\bZ_t$ converges $r-$quickly to $\bZ$ if and only if $\mathbb{E}[L_\epsilon^r],\infty$, where,  $L_\epsilon \coloneqq \sup\left\{ t\in \mathbb{N}: \|\bZ_t-\bZ\|>\epsilon \right\}$. 
\end{definition}

\begin{lemma} \label{lemma: rquick-general}
    Consider $r>0$ and $\epsilon>0$. Let $\left\{\bZ_t\in \mathbb{R}^d: t\in \mathbb{N}\right\}$ be the sequence of random vectors. Then $\bZ_t$ converges $r-$quickly to $\bZ$ if $\sum_{t=1}^\infty t^{r-1}\mathbb{P}\left[\sup_{s\geq t}\|\bZ_s-\bZ\|>\epsilon\right]<\infty$.
\end{lemma}
\begin{proof}
    \begin{align}
          \sum_{t=1}^\infty t^{r-1}&\mathbb{P}\left[\sup_{s\geq t}\|\bZ_s-\bZ\|>\epsilon\right] \\
          &\implies \sum_{t=1}^\infty t^{r-1}\mathbb{P}\left[\sup_{s\geq t}\|(\bZ_s-\bZ)_i\|>\epsilon\right] < \infty, \ \ \forall i \in [d] \\
          &\overset{(a)}{\implies} (\bZ_t)_i \text{ converges $r-$ quickly to $(\bZ)_i$}, \ \  \forall i \in [d] \\
          &\overset{(b)}{\implies} \mathbb{E}\left[ \sup\left\{ t \in \mathbb{N}: |(\bZ_t)_i-(\bZ)_i|>\epsilon  \right\} \right] < \infty, \ \ \forall i \in [d] \\
          &\implies \mathbb{E}\left[ \sup\left\{ t \in \mathbb{N}: \|\bZ_t-\bZ\|>\epsilon  \right\} \right] \leq \mathbb{E}\left[ \sup\left\{ t \in \mathbb{N}: \exists i \in [d],|(\bZ_t)_i-(\bZ)_i|>\epsilon  \right\} \right]\\
          &\hspace{5cm} \leq \sum_{i=1}^d \mathbb{E}\left[ \sup\left\{ t \in \mathbb{N}: |(\bZ_t)_i-(\bZ)_i|>\epsilon  \right\} \right] < \infty \\
          &\implies \bZ_t \text{ converges $r-$ quickly to $\bZ$}.
      \end{align}
      $(a)$ and $b$ follows from Lemma 2 and Definition 2 respectively in \cite{tartakovsky2023quick}. Hence proved.
\end{proof}

\begin{lemma} \label{lemma: rquick-iid}
    Consider $r>0$ and $\epsilon>0$. Let $\left\{\bY_t\in \mathbb{R}^d: t\in \mathbb{N}\right\}$ be the sequence of zero mean i.i.d. random vectors. Let $\overline{\bY_t} = \frac{\sum_{s=1}^t\bY_t}{t}$ denotes the empirical mean. Then the following statements are equivalent. 
    \begin{enumerate}
        \item $\overline{\bY_t}$ converges $r-$quickly to $0$.
        \item $\mathbb{E}\left[ \| \bY \|^{r+1} \right]<\infty$.
        \item $\displaystyle \sum_{t=1}^\infty t^{r-1}\mathbb{P}\left[\sup_{s\geq t}\|\overline{\bY_s}\|>\epsilon\right]<\infty$.
    \end{enumerate}
\end{lemma}
\begin{proof}
      \textbf{Proof of 1) $\implies$ 2)}
      \begin{align}
          1) &\implies \mathbb{E}\left[ \sup\left\{ t \in \mathbb{N}: \|\overline{\bY}_t\|>\epsilon  \right\} \right] < \infty \\
          &\implies \mathbb{E}\left[ \sup\left\{ t \in \mathbb{N}: |(\overline{\bY}_t)_i|>\epsilon  \right\} \right] < \infty, \ \ \forall i \in [d] \\
          &\overset{(a)}{\implies} (\overline{\bY}_t)_i \text{ converges $r-$ quickly to $0$}, \ \  \forall i \in [d] \\
          &\overset{(b)}{\implies} \mathbb{E}\left[ |(\bY)_i|^{r+1} \right]<\infty, \ \ \forall i \in [d] \\
          &\implies \mathbb{E}\left[ \| \bY \|^{r+1} \right] \leq \sum_{i=1}^d \mathbb{E}\left[ |(\bY)_i|^{r+1} \right] < \infty \\
          &\implies 2)
      \end{align}
      \textbf{Proof of 2) $\implies$ 3)}
      \begin{align}
          2) &\implies \mathbb{E}\left[ \| \bY \|^{r+1} \right] < \infty \\
          &\implies \mathbb{E}\left[ |(\bY)_i|^{r+1} \right]<\infty, \ \ \forall i \in [d] \\
          &\overset{(b)}{\implies} \sum_{t=1}^\infty t^{r-1}\mathbb{P}\left[\sup_{s\geq t}\|(\overline{\bY_s})_i\|>\epsilon\right]<\infty, \ \ \forall i \in [d] \\
          &\implies \sum_{t=1}^\infty t^{r-1}\mathbb{P}\left[\sup_{s\geq t}\|\overline{\bY_s}\|>\epsilon\right] 
          \leq \sum_{t=1}^\infty t^{r-1}\mathbb{P}\left[\exists i \in [d]:\sup_{s\geq t}\|(\overline{\bY_s})_i\|>\epsilon\right] \\
          &\hspace{4.8cm} \leq \sum_{i=1}^d \sum_{t=1}^\infty t^{r-1}\mathbb{P}\left[\sup_{s\geq t}\|(\overline{\bY_s})_i\|>\epsilon\right] 
          < \infty\\
          &\implies 3)
      \end{align}
      \textbf{Proof of 3) $\implies$ 1)} follows directly from Lemma \ref{lemma: rquick-general}.\\
      $(a)$ and $(b)$ holds from Definition 2 and Theorem 3 respectively in \cite{tartakovsky2023quick}. Now, the proof is complete.
\end{proof}

\begin{lemma} \label{lemma: mle-rquick}
    Let $\left\{\bZ_t\in \mathbb{R}^d: t\in \mathbb{N}\right\}$ be the sequence of i.i.d. zero mean random vectors which follows parametric distribution with parameter $\btheta$. Assume $\mathbb{E}\left[ \|\nabla\log\mathbb{P}[\bZ\mid \btheta]\|^{r+1} \right]<\infty$ and $\lambda_{\text{min}}\left(-\nabla^2\log\mathbb{P}[\bZ\mid \brho]\right)\geq \sigma^2$, for all $\brho\in \Theta$, $\bZ \in \Omega$, for some $\sigma^2>0$. Define $L_t(\brho) = \sum_{s=1}^t \log\mathbb{P}[\bZ_s\mid \brho]$. Let $\btheta_t$ be the Maximum Likelihood Estimate (MLE) and is defined as $\btheta_t \coloneqq \argmax_{\rho \in \Theta} L_t(\brho)$. Then, then MLE $\btheta_t$ converges $r-$ quickly to the true parameter $\btheta$. 
\end{lemma}
\begin{proof}
    By the mean value theorem (MVT) in multivariate form, there exists $\Tilde{\btheta}$ in the line segment between the points $\btheta$ and $\btheta_t$, such that, 
    \begin{equation}
        \nabla L_t\left( \btheta_t \right) = \nabla(\btheta) + \nabla^2L_t(\Tilde{\btheta})\left( \btheta_t - \btheta \right).
    \end{equation}
    Since $\btheta_t$ is MLE, $\nabla L_t(\btheta_t) = 0$. Hence, on simplification, we get, 
    \begin{equation}
        \btheta_t - \btheta = - \left[ \frac{1}{t}\nabla^2L_t(\Tilde{\btheta}) \right]^{-1} \frac{1}{t}\sum_{s=1}^t\nabla\log\mathbb{P}[\bZ_s\mid \btheta].
    \end{equation}
    Now, we take norm on both sides and upper bound it as follows.
    \begin{align}
        \|\btheta_t - \btheta\| &= \left\| - \left[ \frac{1}{t}\nabla^2L_t(\Tilde{\btheta}) \right]^{-1} \frac{1}{t}\sum_{s=1}^t\nabla\log\mathbb{P}[\bZ_s\mid \btheta] \right\| \\
        &\leq \left\|- \left[ \frac{1}{t}\nabla^2L_t(\Tilde{\btheta}) \right]^{-1}\right\| \left\|\frac{1}{t}\sum_{s=1}^t\nabla\log\mathbb{P}[\bZ_s\mid \btheta]\right\| \\
        &\leq \left[ \lambda_{\text{min}}\left(\frac{1}{t}\nabla^2\sum_{s=1}^t \log\mathbb{P}[\bZ_s\mid \Tilde{\btheta}]\right) \right]^{-1} \left\|\frac{1}{t}\sum_{s=1}^t\nabla\log\mathbb{P}[\bZ_s\mid \btheta]\right\| \\
        &\leq \left[ \frac{1}{t}\sum_{s=1}^t \nabla^2\log\mathbb{P}[\bZ_s\mid \Tilde{\btheta}] \right]^{-1} \left\|\frac{1}{t}\sum_{s=1}^t\nabla\log\mathbb{P}[\bZ_s\mid \btheta]\right\| \\
        &\leq \frac{1}{\sigma^2} \left\|\frac{1}{t}\sum_{s=1}^t\nabla\log\mathbb{P}[\bZ_s\mid \btheta]\right\|.
    \end{align}
    Now we use the above inequality to write the following.
    \begin{equation} \label{eq: 21}
        \mathbb{P}\left[\sup_{u\geq t} \| \btheta_t - \btheta \|>\epsilon\right] \leq \mathbb{P}\left[ \sup_{u\geq t} \left\|\frac{1}{u}\sum_{s=1}^u\nabla\log\mathbb{P}[\bZ_s\mid \btheta]\right\|>\sigma^2\epsilon \right].
    \end{equation}
    It can be verified that $\mathbb{E}\left[\nabla\log\mathbb{P}[\bZ\mid \btheta]\right] = 0$. Hence, $\left\{\nabla\log\mathbb{P}[\bZ_t\mid \btheta]: t \in \mathbb{N}\right\}$ is the sequence of i.i.d. zero mean random vectors.
    We also assume that $\mathbb{E}\left[ \|\nabla\log\mathbb{P}[\bZ\mid \btheta]\|^{r+1} \right]<\infty$. Hence, from Lemma \ref{lemma: rquick-iid}, we can write
    \begin{equation} \label{eq: 22}
        \sum_{t=1}^\infty t^{r-1} \mathbb{P}\left[ \sup_{u\geq t} \left\|\frac{1}{u}\sum_{s=1}^u\nabla\log\mathbb{P}[\bZ_s\mid \btheta]\right\|>\sigma^2\epsilon \right] < \infty.
    \end{equation}
    From equations \eqref{eq: 21} and \eqref{eq: 22}, we get, 
    \begin{equation}
        \sum_{t=1}^\infty t^{r-1} \mathbb{P}\left[\sup_{u\geq t} \| \btheta_t - \btheta \|>\epsilon\right] < \infty.
    \end{equation}
    Now, from Lemma \ref{lemma: rquick-general}, we say $\btheta_t$ converges $r-$quickly to $\btheta$. Hence proved.
\end{proof}

\textbf{Proof of Lemma \ref{lemma: parameter converge}}
\begin{proof}
    Define $\displaystyle T_m^\epsilon \coloneqq \sup\left\{ N_m(t) \in \mathbb{N}: \|\hat{\btheta}_m(t) - \btheta_m\|>\epsilon \right\}$. Since $\hat{\btheta}_m(t)$ is the MLE, from Lemma \ref{lemma: mle-rquick} and from Assumption, we have $\mathbb{E}\left[(T_m^\epsilon)^2\right]<\infty$. Let $t = N_m$ be the time at which arm $m$ observes $T_m^\epsilon$ samples. 

    Since the algorithm uses forced exploration, $N_m(t)\geq \sqrt{\frac{t}{M}}-1$. Hence, we say $T_m^\epsilon\geq \sqrt{\frac{N_m}{M}} - 1$ and on simplification, it yields $N_m\leq 4M\left(T_m^\epsilon\right)^2$ for $N_m\geq 4M$. Hence, we have $\mathbb{E}[N_m]\leq \max\left\{\mathbb{E}\left[4M\left(T_m^\epsilon\right)^2\right], 4M\right\}<\infty$ for all $m \in [M]$.

    We have $N_\epsilon^S = \max_{m \in [M]} N_m \leq \sum_{m=1}^M N_m$. From the above discussion, we have $\mathbb{E}[N_\epsilon^S]\leq \sum_{m=1}^M\mathbb{E}[N_m]<\infty$. Hence proved.
\end{proof}

\section{Proof of Lemma 3} \label{appsec: lemma3}
To prove Lemma \ref{prop: armpullpropconverge}, we do the following.
\begin{enumerate}
    \item We prove that the aggregate gradient $\overline{\bw}(t)$ converges to the optimal arm pull proportions $\bw^{*}\in \mathcal{S}^{*}(\btheta)$ (Lemma \ref{lemma: gradconverge}).
    \item Define the set of over sampled arms $\mathcal{O}_t^{\epsilon} \coloneqq \left\{ m \in [M]: \frac{N_m(t)}{t}> \bw^{*}+\epsilon, \text{ for all } \bw^{*} \in \mathcal{S}^{*}(\btheta) \right\}$. We prove that if the sampling rule selects an arm $m$ such that $\frac{N_m(t)}{t}<\overline{w}_m(t)$, then there exists a stochastic time after which the set $\mathcal{O}_t^\epsilon=\emptyset$ (Lemma \ref{lemma: emptyO}).
    \item We prove that if the set $\mathcal{O}_t^\epsilon=\emptyset$ after some time, then the empirical arm pull proportion converges to the optimal arm pull proportions (Lemma \ref{lemma: ifempyO}).
\end{enumerate}
Finally, we combine Lemmas \ref{lemma: emptyO} and \ref{lemma: ifempyO}, and use the fact that the EBC sampling rule uses gradient tracking, to prove Lemma \ref{prop: armpullpropconverge}

\begin{lemma} \label{lemma: gradconverge}
    In the proposed EBC algorithm, for any given $\epsilon>0$, there exists $M_\epsilon^S \in \mathbb{N}$ such that $\mathbb{E}\left[M_\epsilon^S\right]<\infty$ and for all $t>M_\epsilon^S$, we have, 
    \begin{equation}
        \left| \overline{w}_m(t)-w^{*}_m \right|<\frac{\epsilon}{4} \text{,  } \forall m \in [M] \text{,  for some } \bw^{*} \in \mathcal{S}^{*}(\btheta).
    \end{equation}
\end{lemma}
\begin{proof}
    Let us consider a arbitrary $\bw^{*} \in \mathcal{S}^{*}(\btheta)$.
    Note that $\Gamma_t(\cdot)$ is the minimization of affine functions. Hence, $\Gamma_t(\cdot)$ is a concave function. First, we write the following.
    \begin{align}
        \psi(\bw^{*}, \hat{\btheta}(t)) - \psi(\bw(t), \hat{\btheta}(t)) &\leq g_t^T(\bw(t))(\bw^{*} - \bw(t))\\
        &= \frac{1}{\eta} \left( \bw^{'}(t+1)- \bw(t) \right)^T(\bw^{*} - \bw(t)) \\
        &= \frac{1}{2\eta} \left( \| \bw(t)-\bw^{'}(t+1) \|^2 + \| \bw(t)-\bw^{*} \|^2 - \| \bw^{'}(t+1)-\bw^{*} \|^2 \right) \\
        &= \frac{\eta}{2}\| g_t(\bw(t)) \|^2 + \frac{1}{2\eta} \left( \| \bw(t)-\bw^{*} \|^2 - \| \bw^{'}(t+1)-\bw^{*} \|^2 \right) \\
        &\leq \frac{\eta}{2}\| g_t(\bw(t)) \|^2 + \frac{1}{2\eta} \left( \| \bw(t)-\bw^{*} \|^2 - \| \bw(t+1)-\bw^{*} \|^2 \right) \label{eq: 31}
    \end{align}
    We write the LHS of the above equations as follows.
    \begin{equation}
        \psi\left(\bw^{*}, \hat{\btheta}(t)\right) - \psi\left(\bw(t), \hat{\btheta}(t)\right) = \psi\left(\bw^{*}, \hat{\btheta}(t)\right) - \psi\left(\bw^{*}, \btheta\right) + \psi\left(\bw^{*}, \btheta\right) - \psi\left(\bw(t), \btheta\right) + \psi\left(\bw(t), \btheta\right) - \psi\left(\bw(t), \hat{\btheta}(t)\right)
    \end{equation}
    The above equation can be rewritten and subsequently upper-bounded using \eqref{eq: 31} as follows.
    \begin{align}
        \psi\left(\bw^{*}, \btheta\right) - \psi\left(\bw(t), \btheta\right) &= \psi\left(\bw^{*}, \hat{\btheta}(t)\right) - \psi\left(\bw(t), \hat{\btheta}(t)\right) - \psi\left(\bw^{*}, \hat{\btheta}(t)\right) + \psi\left(\bw^{*}, \btheta\right) - \psi\left(\bw(t), \btheta\right) + \psi\left(\bw(t), \hat{\btheta}(t)\right)\\
        &\leq \frac{\eta}{2}\| g_t(\bw(t)) \|^2 + \frac{1}{2\eta} \left( \| \bw(t)-\bw^{*} \|^2 - \| \bw(t+1)-\bw^{*} \|^2 \right) \\
        &\hspace{2cm} + \left| \psi\left(\bw^{*}, \hat{\btheta}(t)\right) - \psi\left(\bw^{*}, \btheta\right) \right| + \left| \psi\left(\bw(t), \btheta\right) - \psi\left(\bw(t), \hat{\btheta}(t)\right) \right|.
    \end{align}
    From Lemma \ref{lemma: parameter converge} and \ref{lemma: psicont}, for every $\epsilon_1>0$, there exist $N_0$ with $\mathbb{E}[N_0]<\infty$, such that, for $t>N_0$, $\left| \psi\left(\bw, \hat{\btheta}(t)\right) - \psi\left(\bw, \btheta\right) \right|<\frac{\epsilon_1}{4}$, for any $\bw \in \mathcal{P}_M$. Hence, for $t>N_0$, the above equation can be further upper-bounded as follows.
    \begin{align}
        \psi\left(\bw^{*}, \btheta\right) - \psi\left(\bw(t), \btheta\right) &\leq \frac{\eta}{2}\| g_t(\bw(t)) \|^2 + \frac{1}{2\eta} \left( \| \bw(t)-\bw^{*} \|^2 - \| \bw(t+1)-\bw^{*} \|^2 \right) + \frac{\epsilon_1}{2}. 
    \end{align}
    We know each component of the gradient $g_t(\cdot)$ involves the computation of the KL-divergence. From our assumption that the KL-divergence is uniformly continuous and the parameter space is compact, we can say that the KL-divergence is finite. Hence, the norm of the gradient is finite. Let us bound it by $L$, i.e., $\|g_t(\cdot)\|\leq L$. Now, for $t>N_0$, we have, 
    \begin{align}
        \psi\left(\bw^{*}, \btheta\right) - \psi\left(\bw(t), \btheta\right) &\leq \frac{\eta L^2}{2} + \frac{1}{2\eta} \left( \| \bw(t)-\bw^{*} \|^2 - \| \bw(t+1)-\bw^{*} \|^2 \right) + \frac{\epsilon_1}{2}. 
    \end{align}
    We consider the batch size of $B(t)$. On summing the above expression from $t-B(t)+1$ to $t$, and dividing it by $B(t)$, we get the following inequality. Hence for $t>N_0+B(t)-1$
    \begin{align}
        \psi\left(\bw^{*}, \btheta\right) - \frac{1}{B(t)}\sum_{s = t-B(t)+1}^t \psi\left(\bw(s), \btheta\right) &\leq \frac{\eta L^2}{2} + \frac{1}{2\eta B(t)} \left( \| \bw(t-B(t)+1)-\bw^{*} \|^2 - \| \bw(t+1)-\bw^{*} \|^2 \right) + \frac{\epsilon_1}{2} \\
        &\leq \frac{\eta L^2}{2} + \frac{1}{2\eta B(t)}  {\underbrace{\| \bw(t-B(t)+1)-\bw^{*} \|}_{\leq\sqrt{2}}}^2  + \frac{\epsilon_1}{2} \\
        &\leq \frac{\eta L^2}{2} + \frac{1}{\eta B(t)}  + \frac{\epsilon_1}{2}.
    \end{align}
    Since $\psi(\cdot, \btheta)$ is a concave function, we have, 
    \begin{align}
        \psi\left(\bw^{*}, \btheta\right) -  \psi\left(\frac{1}{B(t)}\sum_{s = t-B(t)+1}^t\bw(s), \btheta\right) &\leq \frac{\eta L^2}{2} + \frac{1}{\eta B(t)}  + \frac{\epsilon_1}{2} \\
        \implies \psi\left(\bw^{*}, \btheta\right) -  \psi\left(\overline{\bw}(t), \btheta\right) &\leq \frac{\eta L^2}{2} + \frac{1}{\eta B(t)}  + \frac{\epsilon_1}{2}
    \end{align}
    In the algorithm, we choose $\eta = \frac{c}{\sqrt{B(t)}}$, for some constant $c$. Therefore, for $t>N_0+B(t)-1$, we have, 
    \begin{equation}
        \psi\left(\bw^{*}, \btheta\right) -  \psi\left(\overline{\bw}(t), \btheta\right) \leq \frac{1}{ \sqrt{B(t)}}\left[c\left(\frac{L^2}{2}+1\right)\right]  + \frac{\epsilon_1}{2}.
    \end{equation}
    Let $B(t) = \epsilon_2 t$, for some $0<\epsilon_2<1$. Let $N_3^S = \max\left\{ \frac{N_0-1}{1-\epsilon_2}, \frac{1}{\epsilon_2}\left(\frac{2}{\epsilon_1}\frac{1}{c\left(\frac{L^2}{2}+1\right)}\right)^2 \right\}$ and can be verified that $\mathbb{E}[N_3^S]<\infty$.
     Therefore, for $t>N_3^S$, 
     \begin{equation}
         \psi\left(\bw^{*}, \btheta\right) -  \psi\left(\overline{\bw}(t), \btheta\right) \leq \epsilon_1.
     \end{equation}
     Since $\bw^{*}$ is the maximizer of $\psi(\cdot, \btheta)$, we have, for any $w^{*}\in \mathcal{S}^{*}(\btheta)$, for every $\epsilon_1>0$ there exists $N_3^S$, such that for all $t>N_3^S$, 
     \begin{equation} \label{eq: 32}
         0\leq \psi\left(\bw^{*}, \btheta\right) -  \psi\left(\overline{\bw}(t), \btheta\right) \leq \epsilon_1.
     \end{equation}
     Since $\psi(\cdot, \btheta)$ is a concave continuous function, with compact domain, and $\mathcal{S}^{*}(\btheta)$ is the maximizer set, which is compact and non-empty, for every $\epsilon>0$, there exists $\epsilon_1>0$ such that,
     \begin{equation} \label{eq: 33}
         \text{if } \left|\psi\left(\bw^{*}, \btheta\right) -  \psi\left(\bw, \btheta\right)\right| \leq \epsilon_1, \text{ then, } \left| \bw_m^{*} - \bw \right| < \frac{\epsilon}{4} \text{ for some, } \bw^{*}\in \mathcal{S}^{*}(\btheta).
     \end{equation}
     From equations \eqref{eq: 32} and \eqref{eq: 33}, for every $\epsilon>0$, there exists $N_3^S$, with $\mathbb{E}[N_3^S]$, such that for all $t>N_3^S$, 
     \begin{equation}
         \left| \bw_m^{*} - \overline{\bw}_m(t) \right| < \frac{\epsilon}{4}.
     \end{equation}
     Hence proved.
\end{proof}

\begin{lemma} \label{lemma: emptyO}
    Assume that the sampling strategy in the algorithm satisfies $\frac{N_{a_{t+1}}}{t}<\overline{\bw}_{a_{t+1}}(t)$, for all $t>M_0^{S}$ with $\mathbb{E}[M_0]<\infty$. Define $\mathcal{O}_t^{\epsilon} \coloneqq \left\{ m \in [M]: \frac{N_m(t)}{t}> \bw^{*}+\epsilon, \text{ for all } \bw^{*} \in \mathcal{S}^{*}(\btheta) \right\}$. Then there exists $M_{1, \epsilon}^S$, with $\mathbb{E}[M_{1, \epsilon}^S]<\infty$, such that for all $t>M_{1, \epsilon}^S$, $\mathcal{O}_t ^\epsilon=\emptyset$.
\end{lemma}
\begin{proof}
    From Lemma \ref{lemma: gradconverge}, we have for every $\epsilon>0$, there exists a stochastic time $N_0^{S}$, with $\mathbb{E}[N_0^S]<\infty$, such that for all $t>N_0^S$, $\left| \overline{\bw}_m(t) - \bw^{*} \right|<\frac{\epsilon}{4}$, for some $\bw^{*} \in \mathcal{S}^{*}(\btheta)$. Define $N_1^{S}\coloneqq \left\lceil \frac{8}{\epsilon}-1 \right\rceil$ and $N^{S}\coloneqq \max\left\{ N_1^{S}, N_2^{S} \right\}$. Now, we prove that for all $t>N^{S}$, $\mathcal{O}_t^{\epsilon}=\emptyset$.\\
    \textbf{Case 1:} Assume $\mathcal{O}_{N^{S}}^{\epsilon} = \emptyset$. Now, we show by induction that, for all $t>N^{S}$, $\mathcal{O}_t^{\epsilon}=\emptyset$.\\
    Let us assume that $\mathcal{O}_t^{\epsilon}=\emptyset$. Now, we will show that $\mathcal{O}_{t+1}^\epsilon=\emptyset$. \\
    Consider arm $m\neq a_{t+1}$. We observe that $\frac{N_{m}(t+1)}{t+1} = \frac{N_m(t)}{t+1}<\frac{N_m(t)}{t}$. Hence $m \notin \mathcal{O}_{t+1}^{\epsilon}$.\\
    Consider arm $m = a_{t+1}$. We have the following.
    \begin{align}
        \frac{N_m(t+1)}{t+1} &= \frac{N_m(t)+1}{t+1} \label{eq: 41}\\
        &< \frac{N_m(t)}{t} + \frac{1}{t+1} \\
        &< \overline{w}_m(t) + \frac{1}{t+1} \text{  (sampling rule of EBC)} \\
        &< w^{*}_m + \frac{\epsilon}{4} + \frac{1}{t+1}, \text{for some $\bw^{*}\in \mathcal{S}^{*}(\btheta)$} \text{  $(\because t>N_0^{S})$} \\
        &\leq w^{*}_m + \frac{\epsilon}{2}, \text{for some $\bw^{*}\in \mathcal{S}^{*}(\btheta)$} \text{  $(\because t>N_1^{S})$} \label{eq: 42}
    \end{align}
    Hence $m\notin \mathcal{O}_{t+1}^\epsilon$.\\
    \textbf{Case 2:} Assume $\left| \mathcal{O}_{N^{S}}^\epsilon \right| \geq 1$.\\
    If $m\notin \mathcal{O}_{N^S}^\epsilon$, then $m\notin \mathcal{O}_t^\epsilon$, for all $t>N^S$. It can be proved using the same proof step as Case 1.\\
    Consider $m\in \mathcal{O}_{M^S}^\epsilon$. We claim $m$ is never sampled for $t>N^S$. We prove this by contradiction. Let $a_{t+1} = m$. By using the same steps from equations \eqref{eq: 41} to \eqref{eq: 42}, we get $m\notin \mathcal{O}_{t+1}^\epsilon$. This is a contradiction. Therefore, $m$ will never be sampled for $t>N^S$. 
    For any $j\in \mathcal{O}_{N^S}^\epsilon$, there exists $L_j^\epsilon$, such that for all $t>L_j^\epsilon$, $j \notin \mathcal{O}_t^\epsilon$. Define $L^{\epsilon} \coloneqq \max_{j\in \mathcal{O}_{N^S}^\epsilon}L_j^\epsilon$. Now, for all $t>M_{1, \epsilon}^S\coloneqq \max\{N^S, L^\epsilon\}$, $\mathcal{O}_t^\epsilon=\emptyset$. Hence proved.
\end{proof}

\begin{lemma} \label{lemma: ifempyO}
    If there exists a stochastic time $M_{\epsilon}^S$ satisfying $\mathbb{E}[M_{\epsilon}^S]<\infty$ such that for all $t>M_\epsilon^S$, $\mathcal{O}_t^\epsilon=\emptyset$. Then, for all $t>M_{\frac{\epsilon}{M}}^S$, the allocation of the arms satisfies $d_{\infty}\left( \frac{N_m(t)}{t}, \mathcal{S}^{*}(\btheta) \right)<\epsilon$.
\end{lemma}
\begin{proof}
    Consider $t>M_{\frac{\epsilon}{M}}^S$. To prove that there exist $\bw^{*} \in \mathcal{S}^{*}(\btheta)$, for all $m \in [M]$, $\left| \frac{N_m(t)}{t} - w^{*}_m \right|<\epsilon$. We prove this by contradiction.\\
    We assume that, for all $m \in [M]$, $\left| \frac{N_m(t)}{t} - w^{*}_m \right|<\epsilon$. We write the following. 
    \begin{align}
        1 &= \sum_{m^{'} \in [M]} \frac{N_{m^{'}}(t)}{t} \\
        &= \sum_{m^{'}\neq m} \frac{N_{m^{'}}(t)}{t} + \frac{N_m(t)}{t}\\
        &\leq \sum_{m^{'}\neq m} \left( w^{*}_{m^{'}} + \frac{\epsilon}{M} \right) + \left(w^{*}_{m^{'}}-\epsilon\right), \text{ for some $\bw^{*}\in \mathcal{S}^{*}(\btheta)$}\\
        &= 1-\frac{\epsilon}{M}.
    \end{align}
    It is a contradiction. So, the assumption is wrong. Hence proved. 
\end{proof}

\textbf{Proof of Lemma \ref{prop: armpullpropconverge}}
\begin{proof}
    Since the EBC algorithm's sampling rule uses gradient tracking, we observe that, if the sampling rule selects arm $m$, then it satisfies, $\frac{N_m(t)}{t}<\overline{w}_m(t)$. Now, the proof follows directly from Lemmas \ref{lemma: emptyO} and \ref{lemma: ifempyO}. 
\end{proof}

\section{Proof of Theorem 2}
To prove Theorem \ref{theorem: deltaPC}, we do the following. 
\begin{enumerate}
    \item First, we derive an upper bound for the maximum of the sum of a sequence of twice differentiable, strictly concave multivariate functions (Lemma \ref{lemma: 1}).
    \item We use Lemma \ref{lemma: 1} to derive an upper bound for the KL divergence between the parameter estimate and the true parameter, which depends on a martingale (Lemma \ref{lemma: 1termklbound}).
    \item Then we use Lemma \ref{lemma: 1termklbound} to prove $\delta$-PC by invoking ville's inequality.
\end{enumerate}

\begin{lemma} \label{lemma: 1}
    Let $\{g_s: s\in [t]\}$ be the sequence of twice differentiable strictly concave functions $g:\Theta\rightarrow \mathbb{R}^d$, where $\Theta$ is a compact subset of $\mathbb{R}^d$. Let $\btheta_t$ be the maximizer and is defined as $\btheta_t\coloneqq \argmax_{\brho\in\Theta} \sum_{s\in [t]}g_s(\brho)$. Let $\brho$ be the random vector in the compact space $\Theta$ with the uniform measure $\eta$. Define $V_t\coloneqq -\sum_{s\in [t]}\nabla^2_{\brho}g_s(\brho)\big|_{\brho=\btheta_t}$. Then for any $\epsilon>0$, we have, 
    \begin{equation}
    \begin{aligned}
        \sum_{s\in [t]}g_s(\btheta_t) \leq &\log\mathbb{E}_{\eta}\left[\exp\left(\sum_{s\in [t]}g_s(\brho)\right)\right] + \frac{d}{2}\log\lambda_{\text{max}}(V_t) - d\log\left(1 - 2Q\left(\epsilon\sqrt{\lambda_{\text{max}}(V_t)}\right)\right) + d\log\left(\frac{\sqrt[d]{|\Theta|}}{\sqrt{2\pi}}\right) \\
        &- \min_{\brho\in \prod_{i=1}^d[(\btheta_t)_i-\epsilon, (\btheta_t)_i+\epsilon]} \left\{\sum_{s\in [t]}\left(g_s(\brho)-g_s(\btheta_t)\right)\right\}.
    \end{aligned}
    \end{equation}
\end{lemma}
\begin{proof}
    We write $\sum_{s\in [t]}g_s(\brho)$ as follows.
    \begin{equation}
        \sum_{s\in [t]}g_s(\brho) = \sum_{s\in [t]}g_s(\btheta_t) - \frac{1}{2}(\brho - \btheta_t)^TV_t(\brho - \btheta_t) + R(\brho, \btheta_t),
    \end{equation}
    where, $R(\brho, \btheta_t) = \sum_{s\in [t]}g_s(\brho) - \sum_{s\in [t]}g_s(\btheta_t) + \frac{1}{2}(\brho - \btheta_t)^TV_t(\brho - \btheta_t)$.
    First by raising both sides by $\exp{}$ and then on taking expectation with respect to the uniform measure $\eta$ on both sides, we get
    \begin{equation}
        \mathbb{E}_{\eta}\left[\exp\left(\sum_{s\in [t]}g_s(\brho)\right)\right] = \frac{1}{|\Theta|}\exp\left(\sum_{s\in [t]}g_s(\btheta_t)\right)\int_{\Theta}\exp\left(- \frac{1}{2}(\brho - \btheta_t)^TV_t(\brho - \btheta_t) + R(\brho, \btheta_t)\right)d\brho.
    \end{equation}
    Let $\epsilon>0$ be small such that the hyper cube around $\btheta_t$ is contained in the compact space $\Theta$. Define $\displaystyle R_t^{\text{min}}\coloneqq \min_{\brho\in \prod_{i=1}^d[(\btheta_t)_i-\epsilon, (\btheta_t)_i+\epsilon]} R(\brho, \btheta_t)$ and $\lambda_t\coloneqq \lambda_{\text{max}}(V_t)$. Now, we can lower bound the above expression as follows.
    \begin{equation}
        \mathbb{E}_{\eta}\left[\exp\left(\sum_{s\in [t]}g_s(\brho)\right)\right] \geq \frac{1}{|\Theta|}\exp\left(\sum_{s\in [t]}g_s(\btheta_t)+R_t^{\text{min}}\right)\int_{\brho\in \prod_{i=1}^d[(\btheta_t)_i-\epsilon, (\btheta_t)_i+\epsilon]}\exp\left(- \frac{1}{2}\lambda_t\|\brho - \btheta_t\|^2 \right)d\brho.
    \end{equation}
    Define $\by \coloneqq (\brho-\btheta_t)\sqrt{\lambda_t}$ and $y_t\coloneqq \epsilon\sqrt{\lambda_t}$. By the change of variables, the above integral can be simplified as follows.
    \begin{equation}
        \begin{aligned}
            \mathbb{E}_{\eta}\left[\exp\left(\sum_{s\in [t]}g_s(\brho)\right)\right] &\geq \frac{1}{|\Theta|}\exp\left(\sum_{s\in [t]}g_s(\btheta_t)+R_t^{\text{min}}\right)\frac{1}{(\sqrt{\lambda_t})^d}\int_{\by\in [-y_t, y_t]^d}\exp\left(- \frac{\|\by\|^2}{2} \right)d\brho \\
            &= \frac{1}{|\Theta|}\exp\left(\sum_{s\in [t]}g_s(\btheta_t)+R_t^{\text{min}}\right)\left(\sqrt{\frac{2\pi}{\lambda_t}}\right)^d \left(1-2Q\left(\epsilon\sqrt{\lambda_t}\right)\right)^d.
        \end{aligned}
    \end{equation}
    Now, taking $\log$ on both sides and on rewritting the above inequality, we get, 
    \begin{equation} \label{eq: 1}
        \sum_{s\in [t]}g_s(\btheta_t) \leq \log\mathbb{E}_{\eta}\left[\exp\left(\sum_{s\in [t]}g_s(\brho)\right)\right] + \frac{d}{2}\log\lambda_t - d\log\left(1 - 2Q\left(\epsilon\sqrt{\lambda_t}\right)\right) + d\log\left(\frac{\sqrt[d]{|\Theta|}}{\sqrt{2\pi}}\right) - R_t^{\text{min}}.
    \end{equation}
    Since $g$ is strictly concave, $V_t$ is positive definite. Hence, we lower bound $R_t^{\text{min}}$ as follows.
    \begin{equation} \label{eq: 2}
    \begin{aligned}
        R_t^{\text{min}} &= \min_{\brho\in \prod_{i=1}^d[(\btheta_t)_i-\epsilon, (\btheta_t)_i+\epsilon]} \left\{\sum_{s\in [t]}g_s(\brho) - \sum_{s\in [t]}g_s(\btheta_t) + \frac{1}{2}(\brho - \btheta_t)^TV_t(\brho - \btheta_t)\right\} \\
        &\geq \min_{\brho\in \prod_{i=1}^d[(\btheta_t)_i-\epsilon, (\btheta_t)_i+\epsilon]} \left\{\sum_{s\in [t]}\left(g_s(\brho) -g_s(\btheta_t)\right)\right\}.
    \end{aligned}
    \end{equation}
    On using the inequality \eqref{eq: 2} in \eqref{eq: 1}, we get the desired result. Hence proved.
\end{proof}

\begin{lemma} \label{lemma: 1termklbound}
    Consider $\epsilon>0$. Let $m\in [M]$ be an arm. Then there exists a non-negative martingale $M_m(t)$ satisfying $\mathbb{E}[M_m(1)]=1$ such that, 
    \begin{equation}
    \begin{aligned}
        N_m(t)d_{\text{KL}}\left(\hat{\btheta}_m(t), \btheta_m\right) \leq &\log M_m(t) + \frac{d}{2}\log\left(N_m(t)\mathcal{I}\left(\hat{\btheta}_m(t)\right)\right) - dW^\epsilon\left(\hat{\btheta}_m(t)\right) + d\log\left(\frac{\sqrt[d]{|\Theta|}}{\sqrt{2\pi}}\right) \\
        & + N_m(t)\sum_{i=1}^d \max\left\{ d_{\text{KL}}\left( \hat{\btheta}_m(t), \hat{\btheta}_m(t)-\epsilon \be_i \right), d_{\text{KL}}\left( \hat{\btheta}_m(t), \hat{\btheta}_m(t)+\epsilon \be_i \right) \right\} , 
    \end{aligned}
    \end{equation}
    where, $\displaystyle W^\epsilon\left(\hat{\btheta}_m(t)\right) = \int_{\Omega^{\otimes N_m(t)}}\log\left(1 - 2Q\left(\epsilon\sqrt{\lambda_{\text{max}}(V_t)}\right)\right)\prod_{s\in[t]:A_s=m} \mathbb{P}\left[\bX_s\mid \hat{\btheta}_m(t)\right] d\mathcal{X}_t^m$. 
\end{lemma}

\begin{proof}
    First, we write $N_m(t)d_{\text{KL}}\left(\hat{\btheta}_m(t), \btheta_m\right)$ as follows.
    \begin{align}
        N_m(t)d_{\text{KL}}\left(\hat{\btheta}_m(t), \btheta_m\right) &= \sum_{s\in [t]: A_s=m} \displaystyle \int_{\Omega} \log\left( \frac{\mathbb{P}\left[\bX_s\mid \hat{\btheta}_m(t)\right]}{\mathbb{P}\left[\bX_s\mid \btheta_m\right]} \right) \mathbb{P}\left[\bX_s\mid \hat{\btheta}_m(t)\right] d\bX \\
        & = \int_{\Omega^{\otimes N_m(t)}} \sum_{s\in [t]: A_s=m} \log\left(\mathbb{P}\left[\bX_s\mid \hat{\btheta}_m(t)\right]\right) \prod_{s\in[t]:A_s=m} \mathbb{P}\left[\bX_s\mid \hat{\btheta}_m(t)\right] d\mathcal{X}_t^m \\
        &\hspace{0.25cm} - \int_{\Omega^{\otimes N_m(t)}} \sum_{s\in [t]: A_s=m} \log\left(\mathbb{P}\left[\bX_s\mid \btheta_m\right]\right) \prod_{s\in[t]:A_s=m} \mathbb{P}\left[\bX_s\mid \hat{\btheta}_m(t)\right] d\mathcal{X}_t^m.
    \end{align}
    Let $g_s(\cdot) = \log\left(\mathbb{P}\left[\bX_s \mid \cdot\right]\right)$. 
    Define $V_t = -\sum_{s\in [t]}\nabla^2_{\brho}\log\left(\mathbb{P}\left[\bX_s \mid \brho\right]\right)\big|_{\brho=\hat{\btheta}_m(t)}$. Since $\hat{\btheta}_m(t)$ is the maximum likelihood estimate (MLE), it is the maximizer of $\max_{\brho \in \Theta}\sum_{s\in [t]: A_s=m}\log\left(\mathbb{P}\left[\bX_s \mid \brho\right]\right)$. Hence, on using Lemma \ref{lemma: 1} in the first term in the above expression, we get the following upper bound. 
    \begin{align}
        N_m(t)d_{\text{KL}}\left(\hat{\btheta}_m(t), \btheta_m\right)
        &\leq \underbrace{\int_{\Omega^{\otimes N_m(t)}} \log\mathbb{E}_{\eta}\left[\exp\left(\sum_{s\in [t]:A_s=m}\log\left(\mathbb{P}\left[\bX_s \mid \brho\right]\right)\right)\right] \prod_{s\in[t]:A_s=m} \mathbb{P}\left[\bX_s\mid \hat{\btheta}_m(t)\right] d\mathcal{X}_t^m}_{\mathrm{I}} \\
        &+\frac{d}{2}\underbrace{\int_{\Omega^{\otimes N_m(t)}}  \log\lambda_{\text{max}}(V_t) \prod_{s\in[t]:A_s=m} \mathbb{P}\left[\bX_s\mid \hat{\btheta}_m(t)\right] d\mathcal{X}_t^m}_{\mathrm{II}} - dW^\epsilon\left(\hat{\btheta}_m(t)\right) + d\log\left(\frac{\sqrt[d]{|\Theta|}}{\sqrt{2\pi}}\right)\\
        &\underbrace{-\int_{\Omega^{\otimes N_m(t)}} \min_{\brho\in \prod_{i=1}^d[(\btheta_t)_i-\epsilon, (\btheta_t)_i+\epsilon]} \left\{\sum_{s\in [t]} \log\left( \frac{\mathbb{P}\left[\bX_s\mid \brho\right]}{\mathbb{P}\left[\bX_s\mid \hat{\btheta}_m(t)\right]} \right)\right\} \prod_{s\in[t]:A_s=m} \mathbb{P}\left[\bX_s\mid \hat{\btheta}_m(t)\right] d\mathcal{X}_t^m}_\mathrm{III} \\
        &\underbrace{- \int_{\Omega^{\otimes N_m(t)}} \sum_{s\in [t]: A_s=m} \log\left(\mathbb{P}\left[\bX_s\mid \btheta_m\right]\right) \prod_{s\in[t]:A_s=m} \mathbb{P}\left[\bX_s\mid \hat{\btheta}_m(t)\right] d\mathcal{X}_t^m}_{\mathrm{IV}}. \label{eq: 14}
    \end{align}
    We upper bound $\mathrm{II}$ as follows.
    \begin{align}
        \mathrm{II} &\overset{(a)}{\leq} \log\left( \int_{\Omega^{\otimes N_m(t)}}  \lambda_{\text{max}}(V_t) \prod_{s\in[t]:A_s=m} \mathbb{P}\left[\bX_s\mid \hat{\btheta}_m(t)\right] d\mathcal{X}_t^m\right)\\
        &\overset{(b)}{\leq} \log\left( \int_{\Omega^{\otimes N_m(t)}}  \sum_{s\in [t]: A_s=m}\lambda_{\text{max}}\left(-\nabla^2_{\brho}\log\left(\mathbb{P}\left[\bX_s \mid \brho\right]\right)\big|_{\brho=\hat{\btheta}_m(t)}\right) \prod_{s\in[t]:A_s=m} \mathbb{P}\left[\bX_s\mid \hat{\btheta}_m(t)\right] d\mathcal{X}_t^m\right) \\
        &\overset{(c)}{=} \log\left( \sum_{s\in [t]: A_s=m} \int_\Omega \lambda_{\text{max}}\left(-\nabla^2_{\brho}\log\left(\mathbb{P}\left[\bX \mid \brho\right]\right)\big|_{\brho=\hat{\btheta}_m(t)}\right)\mathbb{P}\left[\bX\mid \hat{\btheta}_m(t)\right] d\bX \right)\\
        &\overset{(d)}{=} \log\left(N_m(t)\mathcal{I}\left(\hat{\btheta}_m(t)\right)\right) \label{eq: 11}.
    \end{align}
    Since $\log(\cdot)$ is a concave function, $(a)$ follows from Jensen's inequality. $(b)$ follows from the fact that the maximum eigenvalue of the sum of matrices is upper bounded by the sum of the maximum eigenvalues of the matrices. $(c)$ holds, since the sequence of random vectors $\{\bX_s, s \in [t]\}$ are i.i.d.. $(d)$ follows from the definition of fisher information $\mathcal{I}(\cdot)$.\\
    We upper bound $\mathrm{III}$ as follows.
    \begin{align}
        \mathrm{III} &\overset{(a)}{=} -\int_{\Omega^{\otimes N_m(t)}}  \sum_{\substack{s\in [t]:\\A_s=m}} \log\left( \frac{\mathbb{P}\left[\bX_s\mid \sum_{i=1}^d\left[\alpha_i^{-}\left(\hat{\btheta}_m(t)-\epsilon\be_i\right) + \alpha_i^{+}\left(\hat{\btheta}_m(t)+\epsilon\be_i\right)\right] \right]}{\mathbb{P}\left[\bX_s\mid \hat{\btheta}_m(t)\right]} \right) \prod_{\substack{s\in[t]:\\A_s=m}} \mathbb{P}\left[\bX_s\mid \hat{\btheta}_m(t)\right] d\mathcal{X}_t^m \\
        &= -\sum_{s\in [t]:A_s=m} \int_{\Omega} \log\left( \frac{\mathbb{P}\left[\bX\mid \sum_{i=1}^d\left[\alpha_i^{-}\left(\hat{\btheta}_m(t)-\epsilon\be_i\right) + \alpha_i^{+}\left(\hat{\btheta}_m(t)+\epsilon\be_i\right)\right] \right]}{\mathbb{P}\left[\bX\mid \hat{\btheta}_m(t)\right]} \right) \mathbb{P}\left[\bX\mid \hat{\btheta}_m(t)\right] d\bX \\
        &\overset{(b)}{\leq} -N_m(t) \int_{\Omega} \sum_{i=1}^d \left[\alpha_i^{-}\log\left( \frac{\mathbb{P}\left[\bX\mid \left(\hat{\btheta}_m(t)-\epsilon\be_i\right)  \right]}{\mathbb{P}\left[\bX\mid \hat{\btheta}_m(t)\right]} \right)+\alpha_i^{+}\log\left( \frac{\mathbb{P}\left[\bX\mid \left(\hat{\btheta}_m(t)+\epsilon\be_i\right)  \right]}{\mathbb{P}\left[\bX\mid \hat{\btheta}_m(t)\right]} \right)\right] \mathbb{P}\left[\bX\mid \hat{\btheta}_m(t)\right] d\bX \\
        &= N_m(t) \sum_{i=1}^d\left[ \alpha_i^{-}d_{\text{KL}}\left( \hat{\btheta}_m(t), \hat{\btheta}_m(t)-\epsilon \be_i \right) + \alpha_i^{+}d_{\text{KL}}\left( \hat{\btheta}_m(t), \hat{\btheta}_m(t)+\epsilon \be_i \right) \right] \\
        &\leq N_m(t)\sum_{i=1}^d \max\left\{ d_{\text{KL}}\left( \hat{\btheta}_m(t), \hat{\btheta}_m(t)-\epsilon \be_i \right), d_{\text{KL}}\left( \hat{\btheta}_m(t), \hat{\btheta}_m(t)+\epsilon \be_i \right) \right\} \label{eq: 12}.
    \end{align}
    Since the convex hull of the vertices of a hypercube is the hypercube itself, the minimizer in the hypercube can be represented as a convex combination of the vertices. Hence, $(a)$ holds. $(b)$ holds from the assumption that $\log\left(\mathbb{P}\left[\bX_s \mid \cdot\right]\right)$ is concave and then by using Jensen's inequality. \\
    We upper bound the sum of $\mathrm{I}$ and $\mathrm{IV}$ as follows.
    \begin{align}
        \mathrm{I}+\mathrm{IV} &= \int_{\Omega^{\otimes N_m(t)}} \log\mathbb{E}_{\eta}\left[\exp\left(\sum_{s\in [t]:A_s=m}\log\left(\mathbb{P}\left[\bX_s \mid \brho\right]\right)\right)\right] \prod_{s\in[t]:A_s=m} \mathbb{P}\left[\bX_s\mid \hat{\btheta}_m(t)\right] d\mathcal{X}_t^m \\
        & - \int_{\Omega^{\otimes N_m(t)}} \log\left[\exp\left(\sum_{s\in [t]: A_s=m} \log\left(\mathbb{P}\left[\bX_s\mid \btheta_m\right]\right)\right)\right] \prod_{s\in[t]:A_s=m} \mathbb{P}\left[\bX_s\mid \hat{\btheta}_m(t)\right] d\mathcal{X}_t^m\\
        &= \int_{\Omega^{\otimes N_m(t)}} \log\mathbb{E}_{\eta}\left[\exp\left(\sum_{s\in [t]:A_s=m}\log\left(\frac{\mathbb{P}\left[\bX_s \mid \brho\right]}{\mathbb{P}\left[\bX_s\mid \btheta_m\right]}\right)\right)\right] \prod_{s\in[t]:A_s=m} \mathbb{P}\left[\bX_s\mid \hat{\btheta}_m(t)\right] d\mathcal{X}_t^m \\
        &\leq \log M_m(t) \label{eq: 13},  \text{    (Jensen's Inequality)}
    \end{align}
    where, $\displaystyle M_m(t) = \int_{\Omega^{\otimes N_m(t)}} \mathbb{E}_{\eta}\left[\exp\left(\sum_{s\in [t]:A_s=m}\log\left(\frac{\mathbb{P}\left[\bX_s \mid \brho\right]}{\mathbb{P}\left[\bX_s\mid \btheta_m\right]}\right)\right)\right] \prod_{s\in[t]:A_s=m} \mathbb{P}\left[\bX_s\mid \hat{\btheta}_m(t)\right] d\mathcal{X}_t^m$.\\
    \begin{claim} \label{claim: martingale}
        $M_m(t)$ is a martingale satisfying $\mathbb{E}[M_m(1)]=1$.
    \end{claim}
    \begin{proof}
        If $A_t\neq m$, it can be verified that $\mathbb{E}\left[M_m(t)\mid \mathcal{F}_{t-1}\right] = M_m(t-1)$.
        If $A_t=m$, we have
        \begin{align}
            \mathbb{E}[M_m(t)\mid \mathcal{F}_{t-1}] &= M_m(t-1)\mathbb{E}\left[ \int_{\Omega} \mathbb{E}_\eta\left[\frac{\mathbb{P}\left[\bX \mid \brho\right]}{\mathbb{P}\left[\bX\mid \btheta_m\right]}\right]\mathbb{P}\left[\bX\mid \hat{\btheta}_m(t)\right] d\bX \right] \label{eq: 15} \\
            &= M_m(t-1)\mathbb{E}_\eta\left[ \int_{\Omega} \mathbb{E}\left[\frac{\mathbb{P}\left[\bX \mid \brho\right]}{\mathbb{P}\left[\bX\mid \btheta_m\right]}\right]\mathbb{P}\left[\bX\mid \hat{\btheta}_m(t)\right] d\bX \right] \\
            &= M_m(t-1)\mathbb{E}_\eta\left[ \int_{\Omega} \underbrace{\left( \int_\Omega \left[\frac{\mathbb{P}\left[\bX \mid \brho\right]}{\mathbb{P}\left[\bX\mid \btheta_m\right]}\right]\mathbb{P}\left[\bX\mid \btheta_m\right]d\bX\right)}_{=1}\mathbb{P}\left[\bX\mid \hat{\btheta}_m(t)\right] d\bX \right] \\
            &= M_m(t-1)\mathbb{E}_\eta\left[ \underbrace{\int_{\Omega} \mathbb{P}\left[\bX\mid \hat{\btheta}_m(t)\right] d\bX}_{=1} \right] \\
            &= M_m(t-1). \label{eq: 16}
        \end{align}
        If $A_1\neq m$, it can be verified that $\mathbb{E}\left[M_m(1)\right] = 1$. If $A_1=m$, we have
        \begin{align}
            \mathbb{E}[M_m(1)] &=\mathbb{E}\left[ \int_{\Omega} \mathbb{E}_\eta\left[\frac{\mathbb{P}\left[\bX \mid \brho\right]}{\mathbb{P}\left[\bX\mid \btheta_m\right]}\right]\mathbb{P}\left[\bX\mid \hat{\btheta}_m(t)\right] d\bX \right] \\
            &\overset{(a)}{=} 1.
        \end{align}
        $(a)$ holds by following the same procedure from the equations \eqref{eq: 15} to \eqref{eq: 16}.
    \end{proof}
    On substituting the equations \eqref{eq: 11}, \eqref{eq: 12} and \eqref{eq: 13} in \eqref{eq: 14}, and from Claim \ref{claim: martingale}, we get the desired result. Hence proved. 
\end{proof}


\textbf{Proof of Theorem \ref{theorem: deltaPC}}
\begin{proof}
    We can write the probability of error as:
    \begin{equation}
        \begin{aligned}
        &\mathbb{P} \left[ \tau_\delta < \infty \text{ and } \mathcal{C}\left(\hat{\btheta}\left(\tau_\delta\right)\right)\nsim \mathcal{C}(\btheta) \right] \\
        &\leq \mathbb{P}\left[ \exists t \in \mathbb{N} : \left\{ Z(t)>\beta(\delta, t) \text{ and } \mathcal{C}(\hat{\btheta}(t))\nsim \mathcal{C}(\btheta) \right\} \right] \\
        &= \mathbb{P}\left[ \exists t \in \mathbb{N} : \left\{ Z(t)> \beta(\delta, t) \right\} \mid \{\mathcal{C}(\hat{\btheta}(t))\nsim \mathcal{C}(\btheta)\} \right]  \mathbb{P}\left[ \exists t \in \mathbb{N} : \left\{\mathcal{C}(\hat{\btheta}(t))\nsim \mathcal{C}(\btheta) \right\} \right]\\
        &\leq \mathbb{P}\Bigg[\exists t \in \mathbb{N} :  \bigg\{ \inf_{\boldsymbol{\lambda} \in \text{Alt}(\hat{\btheta}(t))}  \sum_{m=1}^M N_m(t) d_{\text{KL}}(\hat{\btheta}_m(t), \blambda_m)  > \beta(\delta, t) \bigg\}\mid \{\mathcal{C}(\boldsymbol{\hat{\theta}}(t))\nsim \mathcal{C}(\btheta)\} \Bigg] .
        \end{aligned}
    \end{equation}
    Given $\mathcal{C}(\hat{\btheta}(t))\nsim \mathcal{C}(\btheta)$, we can say $\btheta \in$ Alt$(\hat{\btheta}(t))$. Hence, $\displaystyle \inf_{\boldsymbol{\lambda} \in \text{Alt}(\hat{\btheta}(t))}  \sum_{m=1}^M N_m(t) d(\hat{\btheta}_m(t) \| \blambda_m) 
    \leq \sum_{m=1}^M N_m(t) d(\hat{\btheta}_m(t) \| \btheta_m)$. 
    Using this inequality, we bound the probability of error as follows.
    \begin{equation} \label{del:eq4}
    \begin{aligned}
        \mathbb{P} \left[ \tau_\delta < \infty \text{ and } \mathcal{C}\left(\hat{\btheta}\left(\tau_\delta\right)\right)\nsim \mathcal{C}(\btheta) \right] 
        &\leq \mathbb{P}\left[ \exists t \in \mathbb{N} : \left\{ \sum_{m=1}^M N_m(t) d(\hat{\btheta}_m(t) \| \btheta_m)> \beta(\delta, t) \right\}  \right].
    \end{aligned}
    \end{equation}
    Using Corollary \ref{lemma: 1termklbound}, we get
    \begin{equation}
    \begin{aligned}
        \sum_{m=1}^M N_m(t) d_m\left( \hat{\theta}_m(t) \| \theta_m \right) \leq &\log\left( M(t) \right) + \frac{d}{2}\sum_{m=1}^M\log\left( N_m(t)\mathcal{I}\left(\hat{\btheta}_m(t)\right) \right) - d\sum_{m=1}^MW^{\epsilon}\left(\hat{\btheta}_m(t)\right) +Md\log\left(\frac{\sqrt[d]{\Theta}}{\sqrt{2\pi}}\right)\\
        &+ \sum_{m=1}^M N_m(t)\sum_{i=1}^d \max\left\{ d\left( \hat{\btheta}_m(t) \| \hat{\btheta}_m(t)-\epsilon\be_i \right), d\left( \hat{\btheta}_m(t) \| \hat{\btheta}_m(t)+\epsilon\be_i \right) \right\} 
    \end{aligned}
    \end{equation}
    where $M(t) = \prod_{m=1}^M M_m(t)$. It can be shown that $M(t)$ is also a non negative martingale satisfying $\mathbb{E}[M(1)] = 1$. 
    Substituting the above equation along with the expression of the threshold $\beta(\delta, t)$ in equation \eqref{del:eq4} we get, 
    \begin{equation}
    \begin{aligned}
        \mathbb{P} \left[ \tau_\delta < \infty \text{ and } \mathcal{C}\left(\hat{\btheta}\left(\tau_\delta\right)\right)\nsim \mathcal{C}(\btheta) \right] &\leq \mathbb{P}\left[ \exists t \in \mathbb{N} \log{M_t} \geq \log{\frac{1}{\delta}} \right] \\
        &\leq \delta \mathbb{E}[M_1] \text{ (Ville's inequality) }\\
        &= \delta.
    \end{aligned}
    \end{equation}
\end{proof}

\section{Proof of Theorem 3}

\begin{proof}
    From Lemma \ref{lemma: parameter converge}, for all $\epsilon_1 >0$, there exists $N_1^S$, with $\mathbb{E}[N_1^S]>0$, such that for all $t>N_1^S$, 
    \begin{equation} \label{eq: 51}
        \left\|\hat{\btheta}_m(t)-\btheta_m\right\|<\epsilon_1, \text{ for all $m \in [M]$}.
    \end{equation}
    From Lemma \ref{prop: armpullpropconverge}, for $\epsilon_1>0$, there exists $N_2^S$, with $\mathbb{E}[N_2^S]>0$, such that for all $t>N_2^S$, 
    \begin{equation} \label{eq: 52}
        \left|\frac{N_m(t)}{t}-w_m^{*} \right|<\epsilon_1, \text{ for all $m \in [M]$, for some $\bw^{*}\in \mathcal{S}^{*}(\btheta)$}
    \end{equation}
    From Lemma \ref{lemma: psicont}, we have the inner infimum function $\psi(\cdot, \cdot)$ is continuous. Hence, for all $\epsilon_2>0$, there exist $\epsilon_1>0$ such that if $\left\|\btheta_m^{'}-\btheta_m\right\|<\epsilon_1$, for all $m\in [M]$ and $\left| w_m^{'}-w_m^{*} \right|<\epsilon_1$, for all $m\in [M]$, for some $\bw^{*}\in \mathcal{S}^{*}(\btheta)$, then 
    \begin{equation} \label{eq: 53}
        \left| \psi\left( \bw^{'}, \btheta^{'} \right) - \psi\left(\bw, \btheta\right) \right| <\epsilon_2.
    \end{equation}
    From equations \eqref{eq: 51}, \eqref{eq: 52} and \eqref{eq: 53}, we have for $\epsilon_2>0$, there exists $N_3^S\coloneqq \max\{N_1^S, N_2^S\}$ such that for all $t>N_3^S$, we have, 
    \begin{equation}
        \left| \psi\left( \frac{\bN(t)}{t}, \hat{\btheta}(t) \right)-\psi\left(\bw^{*}, \btheta\right) \right| < \epsilon_2, \text{ for some $\bw^{*} \in \mathcal{S}^{*}(\btheta)$.}
    \end{equation}
    Let us define $N_5^S\coloneqq \max\{ N_3^S, N_4^S \}$. $N_4^S$ will be defined later in the proof. Now we write the following. 
    \begin{align}
        \tau_\delta-1 &= (\tau_\delta-1)\mathds{1}_{\{\tau_\delta-1\leq N_5^S\}} + (\tau_\delta-1)\mathds{1}_{\{ \tau_\delta-1>N_5^S \}} \\
        &\leq N_5^S + (\tau_\delta-1)\mathds{1}_{\{ \tau_\delta-1>N_5^S \}}. \label{eq: 57}
    \end{align}
    For $\tau_\delta-1>N_5^S$, we have, 
    \begin{align}
        \psi\left( \bw^{*}, \btheta \right) - \epsilon_2 &\leq \psi\left( \frac{\bN(\tau_\delta-1)}{\tau_\delta-1}, \hat{\btheta}(\tau_\delta-1)        \right)\\
        &\leq \frac{\beta\left( \delta, \tau_\delta-1 \right)}{\tau_\delta-1}. \label{eq: 54}
    \end{align}
    Now we upper bound the threshold as follows.
    \begin{equation} \label{eq: 55}
        \beta(\delta, t) \leq \frac{d}{2}\sum_{m=1}^M \log\mathcal{I}\left(\hat{\btheta}_m(t)\right) + \underbrace{\frac{Md}{2}\log\left(\frac{t}{M}\right)}_{\mathrm{I}} + \log\left(\frac{1}{\delta}\right) - d\sum_{m=1}^MW^\epsilon\left(\hat{\btheta}_m(t)\right) + Md\log\left(\frac{\sqrt[d]{|\Theta|}}{\sqrt{2\pi}}\right) + t\eta(\epsilon), 
    \end{equation}
    where $\mathrm{I}$ follows from AM-GM inequality and $\eta(\epsilon)\coloneqq \max_{\btheta\in \Theta}\max_{m\in [M]}\sum_{i=1}^d \max\left\{ d\left( \btheta \| \btheta-\epsilon\be_i \right), d\left( \btheta \| \btheta+\epsilon\be_i \right) \right\} $. On substituting equation \eqref{eq: 55} in \eqref{eq: 54}, we get, 
    \begin{align}
        (\tau_\delta-1)\left[\psi(\bw^{*}, \btheta)-\epsilon_1\right] &\leq \frac{d}{2}\sum_{m=1}^M \log\mathcal{I}\left(\hat{\btheta}_m(\tau_\delta-1)\right) + \frac{Md}{2}\log\left(\frac{\tau_\delta-1}{M}\right) + \log\left(\frac{1}{\delta}\right) - d\sum_{m=1}^MW^\epsilon\left(\hat{\btheta}_m(\tau_\delta-1)\right) \\
        &+ Md\log\left(\frac{\sqrt[d]{|\Theta|}}{\sqrt{2\pi}}\right) + (\tau_\delta-1)\eta(\epsilon).
    \end{align}
    Define $\mathcal{I}_{\text{max}}\coloneqq \max_{\btheta\in \Theta}\mathcal{I}(\btheta)$. Also, from the assumption $\lambda_{\text{min}}\left( -\nabla^2\log\left(\bX\mid\btheta\right) \right)$, we have the following.
    \begin{align}
        -\log\left( 1-2Q\left(\epsilon\sqrt{\lambda_{\text{max}}(V_t)}\right) \right) &\leq -\log\left( 1-2Q\left(\epsilon\sigma\sqrt{\min_{m\in [M]}N_m(t)}\right) \right) \\
        &\leq -\log\left( 1-2Q\left(\epsilon\sigma\sqrt{\sqrt{\frac{t-1}{M}}-1}\right) \right).
    \end{align}
    Define $N_4^S \coloneqq M\left(\frac{1}{\epsilon\sigma}Q^{-1}\left(\frac{1}{4}\right)+1\right)$. It can be shown that for all $t>N_4^S$, we have, 
    \begin{equation}
        -\log\left( 1-2Q\left(\epsilon\sigma\sqrt{\sqrt{\frac{t-1}{M}}-1}\right) \right) \leq \log2.
    \end{equation}
    Hence for all $\tau_\delta-1>N_5^S$, we have
    \begin{equation}
        (\tau_\delta-1)\left[ \psi(\bw^{*}, \btheta)-\epsilon_2 \right] \leq \frac{d}{2}\sum_{m=1}^M \log\mathcal{I}_{\text{max}} + \frac{Md}{2}\log\left(\frac{\tau_\delta-1}{M}\right) + \log\left(\frac{1}{\delta}\right) + Md\log\left(\frac{\sqrt{2}\sqrt[d]{|\Theta|}}{\sqrt{\pi}}\right) + (\tau_\delta-1)\eta(\epsilon)
    \end{equation}
    The above expression can be rewritten as follows.
    \begin{equation}
        \frac{2}{Md}(\tau_\delta-1)\left[ \psi(\bw^{*}, \btheta)-\epsilon_2-\eta(\epsilon) \right] \leq \log\left[ \frac{2\mathcal{I}_{\text{max}}(\tau_\delta-1)|\Theta|^{\frac{2}{d}}}{M\delta^{\frac{2}{Md}}\pi} \right].
    \end{equation}
    Hence for $\tau_\delta-1>N_5^S$, we have, 
    \begin{equation}
        \tau_\delta-1\leq \inf\left\{ t: t\frac{2}{Md}\left[ \psi(\bw^{*}, \btheta)-\epsilon_2-\eta(\epsilon) \right] \geq \log\left[ \frac{2\mathcal{I}_{\text{max}}(t)|\Theta|^{\frac{2}{d}}}{M\delta^{\frac{2}{Md}}\pi} \right]   \right\}
    \end{equation}
    On using Lemma \ref{Lemma: infbound}, we get, 
    \begin{equation}
    \begin{aligned}
        \tau_\delta - 1 \leq \frac{1}{\frac{2}{Md}\left( \psi\left( \bw^{*}, \btheta \right) - \epsilon_2 - \eta(\epsilon) \right)} &\left[ \log\left( \frac{\frac{2\mathcal{I}_{\text{max}}|\Theta|^{\frac{2}{d}} }{\pi M \delta^{\frac{2}{Md}}}e}{\frac{2}{Md}\left( \psi\left( \bw^{*}, \btheta\ \right) - \epsilon_2 - \eta(\epsilon) \right)} \right) \right.\\
        &+ \left.\log{\log\left( \frac{\frac{2\mathcal{I}_{\text{max}}|\Theta|^\frac{2}{d} }{\pi M \delta^{\frac{2}{Md}}}}{\frac{2}{Md}\left( \psi\left( \bw^{*}, \btheta \right) - \epsilon_2 - \eta(\epsilon) \right)} \right)}\right].
    \end{aligned}
    \end{equation}
    Hence, by using the above inequality on equation \eqref{eq: 57}, we get
    \begin{equation} \label{eq:taubound}
    \begin{aligned}
        \tau_\delta \leq N^\epsilon + \frac{1}{\frac{2}{Md}\left( \psi\left( \bw^{*}, \btheta \right) - \epsilon_2 - \eta(\epsilon) \right)} &\left[ \log\left( \frac{\frac{2\mathcal{I}_{\text{max}}|\Theta|^\frac{2}{d} }{\pi M \delta^{\frac{2}{Md}}}e}{\frac{2}{Md}\left( \psi\left( \bw^{*}, \btheta \right) - \epsilon_2 - \eta(\epsilon) \right)} \right) \right.\\
        &+ \left.\log{\log\left( \frac{\frac{2\mathcal{I}_{\text{max}}|\Theta|^\frac{2}{d} }{\pi M \delta^{\frac{2}{Md}}}}{\frac{2}{Md}\left( \psi\left( \bw^{*}, \btheta \right) - \epsilon_2 - \eta(\epsilon) \right)} \right)}\right] + 1.
    \end{aligned}
    \end{equation}
    On taking expectation on the above inequality, dividing it by $\log{\frac{1}{\delta}}$, as $\delta \rightarrow 0$, we get the following.
    \begin{equation}
        \lim_{\delta \rightarrow 0} \frac{\mathbb{E}[\tau_\delta]}{\log{\frac{1}{\delta}}} \leq \frac{1}{\psi\left( \bw^{*}, \btheta \right) - \epsilon_2 - \eta(\epsilon)}.
    \end{equation}
    On letting $\epsilon_2$ tends to $0$, we get, 
    \begin{equation}
        \lim_{\delta \rightarrow 0} \frac{\mathbb{E}[\tau_\delta]}{\log{\frac{1}{\delta}}} \leq \frac{1}{\psi\left( \bw^{*}, \btheta \right)- \eta(\epsilon)}.
    \end{equation}
\end{proof}

\begin{lemma} \label{Lemma: infbound}
         For any constants $c_1, c_2 > 0$ and $\frac{c_2}{c_1}>1$, we have, 
        \begin{equation}
        \begin{aligned}
            \inf\left\{ t \in \mathbb{N}: c_1t \geq \log\left( c_2 t \right) \right\} &\leq \frac{1}{c_1}\left( \log\left( \frac{c_2e}{c_1} \right) + \log\log\left( \frac{c_2}{c_1} \right) \right).
        \end{aligned}
        \end{equation}
    \end{lemma}
    \begin{proof}
        Lemma $8$ in \cite{jedra2020optimal}
    \end{proof}